\documentclass{article}

\usepackage{arxiv}

\usepackage[utf8]{inputenc} 
\usepackage[T1]{fontenc}    
\usepackage{url}            
\usepackage{booktabs}       
\usepackage{amsfonts}       
\usepackage{nicefrac}       
\usepackage{microtype}      
\usepackage{lipsum}		
\usepackage{graphicx}
\usepackage{natbib}
\usepackage{doi}
\usepackage{subfigure}
\usepackage{algorithmic}
\usepackage[linesnumbered,ruled]{algorithm2e}
\newtheorem{myProblem}{Problem}
\newtheorem{definition}{Definition}

\newtheorem{lemma}{Lemma}
\newtheorem{proposition}{Proposition}
\newtheorem{proof}{Proof}
\usepackage{diagbox}
\usepackage{amsmath,amssymb,amsfonts}
\title{Route Optimization via Environment-Aware Deep Network and Reinforcement Learning\thanks{This manuscript has been accepted by \textit{ACM Transactions on Intelligent Systems and Technology} on April 25, 2021.}}

\author{
  Pengzhan Guo \\
  Stony Brook University\\
  Stony Brook, USA, 11794 \\
  \texttt{guopengzhan@hotmail.com} \\
  \And
  Keli Xiao \\
  Stony Brook University\\
  Stony Brook, USA, 11794 \\
  \texttt{keli.xiao@stonybrook.edu} \\
  \And
  Zeyang Ye \\
  Samsung Research America\\
  Mountain View, USA, 94043 \\
  \texttt{zeyang.ye3@gmail.com} \\
  \And
  Wei Zhu \\
  Stony Brook University\\
  Stony Brook, USA, 11794 \\
  \texttt{wei.zhu@stonybrook.edu} \\
}
\renewcommand{\headeright}{Guo et al.}
\renewcommand{\shorttitle}{Route Optimization via Deep Network and Reinforcement Learning}
\begin{document}
\maketitle
\begin{abstract}
Vehicle mobility optimization in urban areas is a long-standing problem in smart city and spatial data analysis.
Given the complex urban scenario and unpredictable social events, our work focuses on developing a mobile sequential recommendation system to maximize the profitability of vehicle service providers (e.g., taxi drivers).
In particular, we treat the dynamic route optimization problem as a long-term sequential decision-making task. 
A reinforcement-learning framework is proposed to tackle this problem, by integrating a self-check mechanism and a deep neural network for customer pick-up point monitoring.
To account for unexpected situations (e.g., the COVID-19 outbreak), our method is designed to be capable of handling related environment changes with a self-adaptive parameter determination mechanism.
Based on the yellow taxi data in New York City and vicinity before and after the COVID-19 outbreak, we have conducted comprehensive experiments to evaluate the effectiveness of our method.
The results show consistently excellent performance, from hourly to weekly measures, to support the superiority of our method over the state-of-the-art methods (i.e., with more than 98$\%$ improvement in terms of the profitability for taxi drivers).
\end{abstract}


\maketitle
\section{Introduction}
Taxicab service plays an essential and irreplaceable role in urban traffic system \citep{ji2020spatio}.
For example, in New York City, there are more than 21,000 taxi drivers and more than 80,000 ride-sharing drivers.
Compared to other means of daily transportation, such as bus and subway, taxis usually offers a better trip experience in terms of comfort, convenience, and travel time accommodation.
Thus, it has been a long-standing central issue to improve the efficiency of vehicle mobility by optimizing the route recommendation for drivers for taxi services in big cities like New York, Tokyo, and Beijing
 \citep{yuan2011find,zheng2014urban}.


Based on large-scale taxi trace data, there is an extensive literature on route recommendation systems.
Some studies focus on the traditional optimization method.
For example, \citet{qu2014cost} proposed a cost-efficient objective function and developed a greedy method to maximize the potential net profit.
Similar methods can be found in \citep{ding2013hunts,zhou2016method}.
Stochastic optimization methods (e.g., simulated annealing -SA-) and parallel computing techniques have also been applied to route recommendation problems to speed up the route searching tasks (see \citep{ye2018applying,ye2018unified,zhang2019parallel}). 
To avoid identical route recommendation to different drivers, \citet{xiao2020multi} developed a multi-user mobile sequential recommendation model with a designed core rotation and mixing strategy to enhance SA when handling the parallel search for multiple drivers.


On the other hand, some studies focus on machine learning-based approaches for route recommendation problems \citep{wang2017taxirec, hu2019vizml,li2013recommendation}.
For example, \citet{wang2017taxirec} deployed a rank-based extreme learning machine (ELM) model to recommend road clusters to taxi drivers for passenger seeking.
By exploring road clusters through a clustering process based on the middle point of the road segment, their method aims to increase the pick-up probability for drivers.
\citet{garg2018route} implemented a Monte Carlo tree search method to minimize the traveling distance for taxi drivers.
\citet{zi2019anomalous} proposed a cloud-based system and applied machine learning algorithms to help passengers detect abnormal taxi trajectories.
Importantly, \citet{rong2016rich} suggested that drivers' long-term passenger seeking process can be viewed as a Markov Decision Process (MDP).
Considering that reinforcement learning (RL) techniques are powerful in handling MDP \citep{mnih2015human}, RL-based approaches have appeared in some recent studies. 
After introducing a comprehensive process of traffic-related feature extraction, \citet{ji2020spatio} applied the classical deep reinforcement learning method to a dynamic route recommendation system.


However, two outstanding issues in the existing route recommendation literature are still under investigation.
First, although the classic version of RL has been shown to be effective in dynamic route searching, few studies can be found to address the adaptive versions of the RL-based method.
Thus, we believe that RL's performance in route recommendation can still be significantly improved via an adaptive design.
Second, most of the existing methods have not been investigated under a dynamic urban environment scenario (e.g., sudden situation changes due to unexpected public health emergence like the COVID-19 outbreak).
In this case, an effective route recommendation system should be able to detect these abnormal situations and conduct related self-adjustments. 

To address these two issues, we propose an adaptive reinforcement learning method with a self-check mechanism.
Our method will not only accelerate the convergence rate of traditional RL methods, but also handle sudden vehicle demand changes due to unpredictable public emergency.
In summary, our work differs from existing methods and contributes to the literature in three ways.
\begin{itemize}
    \item First, we apply a self-check mechanism to periodically compare the current policy to the preserved ones, supported by theoretical analysis. Using the self-check mechanism, our model can achieve a better performance than the classical reinforcement learning method under the same condition.
    \item Second, we use a deep neural network to enrich the ability of our method in detecting potential situation changes. Upon encountering dynamic environments, our model can automatically adjust its parameters to achieve optimal route recommendation performance. The effectiveness of the parameter updating approach is supported by theoretical analysis.
    \item 
     Finally, our method's objective function is flexibly extensible to deal with different aspects of the path quality (e.g., traveling costs, profitability, etc.). Based on data from New York City, we evaluate our method with a focus on the profitability of recommended routes. The results have validated the superiority of our method over other benchmarks, including the state-of-the-art methods. Compared with existing methods, our method can achieve 98\% or more earnings for taxi drivers.
     Importantly, given the sudden situation change caused by COVID-19, our method leads to consistently superior performance in terms of the hourly and weekly income across different months.
\end{itemize}

The rest of the paper is organized as follows. 
In Section~\ref{sec:formulation}, we propose the dynamic taxi route
recommendation problem after laying out related definitions.
Section~\ref{sec:model} introduces key definitions and concepts particular to our model. 
Section~\ref{sec:method} presents our methodology, including detailed explanation of the sequential method, and theoretical discussions of their effectiveness.
In Section~\ref{sec:result}, we demonstrate and then discuss the results on a large-scale real-world dataset. 
Subsequently, we summarize additional related work in Section~\ref{sec:related work} and finally, we conclude in Section~\ref{sec:conclusion and Future WOrk}.

\section{Problem and Proposed Framework}\label{sec:formulation}
In this section, we introduce important definitions and formalize our problem.
\subsection{Preliminary}
\begin{figure}[!]
\centering
   \subfigure {\includegraphics[width=0.8\linewidth]{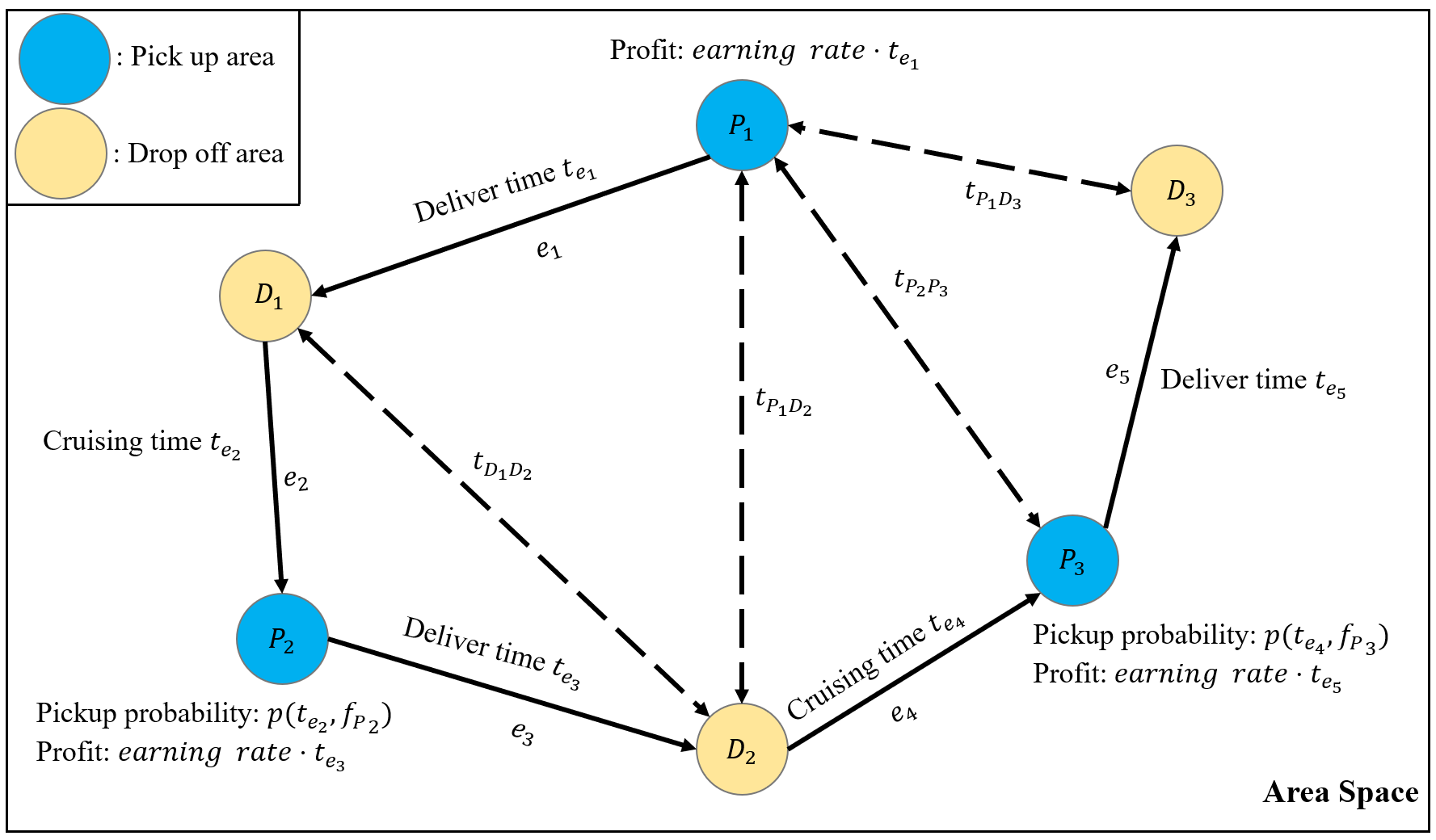}}
   \caption{Road Network.
   }\label{fig:rn}
\end{figure}
\begin{definition}(Road Segment).
A long path can be divided into several road segments by its connections. 
Specifically, each segment is associated with a start point and an end point.
Moreover, if an area is connected to multiple areas, then this area has several road segments.
\end{definition}

\begin{definition}(Route). 
 A recommended route for a taxi driver at a specific time is a sequence of connected road segments, denoted as $r= e_1 \rightarrow  e_2 \rightarrow...\rightarrow  e_k$ where $ e_i \in E$, and $E$ is the set of all the road segments.
\end{definition}

\begin{definition}(Road Network).
Traditionally, a road network is a directed graph.
It is denoted as $G$=<$V,$ $E$>, where $V$ denotes road intersections, and E means road segments.
Since we focus on area recommendation, in our settings, $V$ denotes the area intersections and E refers to the connections between all areas.
For any $e \in E$, there are two vertices: $V_{in}$ and $V_{out}$.
Since we focus on the area recommendation, <$V_{in}$, $V_{out}$> is equivalent to <$V_{out}, V_{in}$>.
The road network $G$ is defined as a city.
\end{definition}
\autoref{fig:rn} depicts an example of a road network with a recommended route by showing the profit.
In the graph, the road segment is two connected areas.
As mentioned before, the time for <$V_{in}, V_{out}$> is the same as <$V_{out}, V_{in}$> where $V$ denotes the point in the graph. 
For the recommended path $r$ that is shown as solid line $ e_1 \rightarrow  e_2 \rightarrow ... \rightarrow e_5$, the edges are one-directional due to the setting that taxi drivers cannot drive back and forth in the same single road segment.
This setting is to avoid causing traffic jams and accidents as mentioned in \citep{qu2014cost}.
The total profit for one area is equal to the multiplication of delivery time and the minute income based on historical data in that area.
The pick-up probability is dynamic and related to the arriving time and pick-up frequency. We will show more details on the pick-up probability in Section~\ref{sec:model}.

\subsection{Problem Statement}
Suppose that a taxi driver is at a location $c_0$, and $R$ represents the set of all possible routes starting at $c_0$. 
The driver can evaluate a recommended route $r_i\in R$ based on his/her demands (e.g., traveling distance, traveling time, profitability, etc.), denoted by a function $g(\cdot)$. The general route optimization problem can be formulated as:
\begin{myProblem}
(General Route Optimization Problem). 
Given $c_0$ the starting location (area) of a taxi driver, we recommend the optimal route $r^*$ to the driver. That is,
\begin{equation}
   r^*=\underset{r_i \in R}{argmax}  \left(g({r_i})\right),
\end{equation}
where $r_i \in R$ is any possible route with the starting location $c_0$; $g(\cdot)$ is the path quality evaluation function.
\end{myProblem}

Note that $g(\cdot)$ can be defined differently, such as the potential traveling distance to find the next passenger \citep{ge2010energy,ye2018applying}, expected traveling time to find the next passenger \citep{ye2018multi,xiao2020multi}, profitability of the recommended route \citep{qu2014cost}, and so forth. 
This paper mainly investigates the route optimization problem by evaluating the profitability, and the reasons are twofold.
First, no matter whether the recommended route will minimize the traveling distance or time,  profitability is always the fundamental user demand \citep{zhou2018optimizing,chen2020framework}.
Second, given the availability of taxi drivers' earning (per minute) data, profitability serves as a direct route quality measurement compared with other metrics.
Thus, we can reformulate the route optimization problem as a profitability-oriented route optimization problem that aims to maximize taxi drivers' income.

\begin{myProblem}
(Profitability-Oriented Route Optimization Problem). 
Given $c_0$ the start location (area) of a taxi driver and the current time, we recommend a route to maximize the earnings of the driver, and the optimal route $r^*$ can be represented as:
\begin{equation}
   r^*=\underset{r_i \in R}{argmax} \sum_{P_j,D_j \in r_i} \left(p_{gr}(P_j,D_j)\cdot INC(P_j)\cdot t_{P_j D_j}\right),
\end{equation}
where $r_i \in R$ is any possible route with the starting location $c_0$; $P_j$ and $D_j$ represent a pick-up and a drop-off area, respectively, and $p_{gr}(\cdot)$ is the pick-up probability; $INC(P_j)$ is the evaluated earning rate given a pick-up area; $t_{P_j D_j}$ is the expected traveling time from $P_j$ to $D_j$.
\end{myProblem}

\begin{figure}[!]
\centering
   \subfigure {\includegraphics[width=0.9\linewidth]{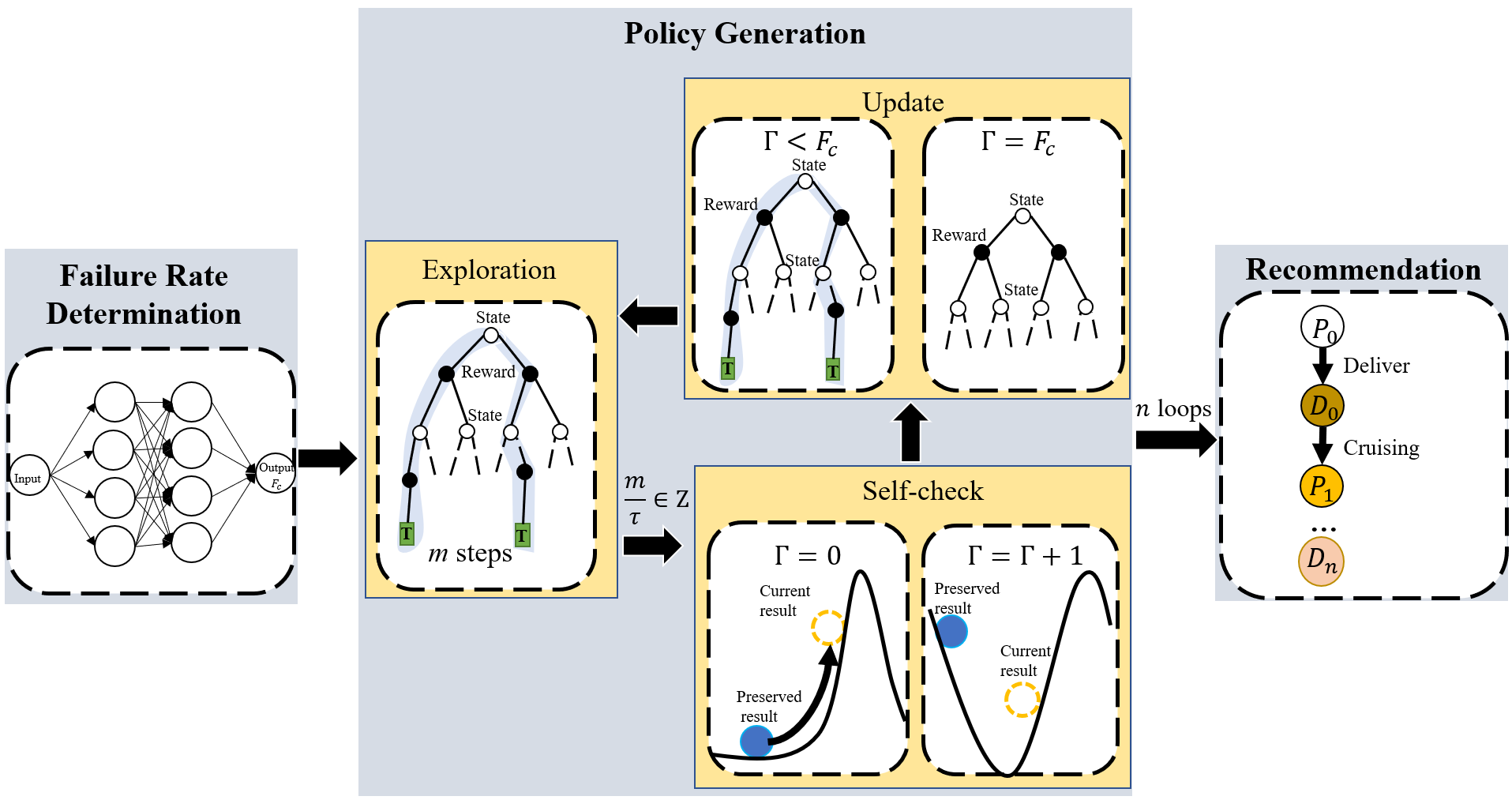}}
   \caption{The Structure of Our Method (ATDSC).
   }\label{fig:ms}
\end{figure}
\subsection{The Structure of the Proposed Method}
To address the problem defined above, we propose an adaptive temporal difference learning with self-check mechanism (ATDSC).
The structure of the ATDSC is demonstrated in \autoref{fig:ms}.
As can be seen, the framework contains two major components, including a neural network for pick-up failure rate determination and a deep-learning component for policy generation. 

Specifically, we first input the current pick-up frequencies and related travel records of the previous year to determine parameters related to the self-check mechanism via a deep neural network.
Then, in the policy generation part, we imply a Temporal Difference (TD) learning process to handle the exploring task for optimal routes.
Once the iterations are divisible by the check parameter $\tau$, the model will stop and check the quality of the current policy.
The self-check process can be viewed as a Markov Chain Process (MCP) whose transition probability is determined by the quality of the policy.
If the quality of the path under the updated policy keeps improving, the transition probability remains zero.
Suppose the quality of the path does not change after several rounds of checks. In that case, the transition probability will be set to one, and the model will be transferred to the original blank situation.
The above exploring strategy is designed to accelerate the convergence rate of our method. 
Finally, the model will output the preserved optimal path in the recommendation part.
We shall discuss the technical details with theoretical analysis in Section ~\ref{sec:model} and Section~\ref{sec:method}.

\section{Reinforcement Learning Framework}\label{sec:model}
In this section, we discuss our reinforcement learning (RL) framework and related design of the reward and transition probability.

\subsection{Temporal Difference Learning}
To form an RL framework for route recommendation, we consider different pick-up areas as different states.
The action $a$ can be defined as the selection of the next pick-up area to visit.
The reward is defined as the profit at pick-up location $P_i$.
Suppose that we have complete knowledge regarding the environment.
Denoted by $q_{\pi}(P_{i},a)$, the action-value function of starting at a pick-up location $P_i \in V$ under a policy $\pi$ can be defined as:
\begin{equation} \label{eq:valuestate}
    \begin{split}
    & q_{\pi}(P_{i},a)~\dot{=}~E_{\pi}\left[\sum_{j=0}^n \gamma^j S_{P_{i+j+1}}| P_t = P_{i}, A_t=a\right],\\
    \end{split}
\end{equation} 
where $\gamma$ is the discount rate; $n$ is the number of states on the path; $E_{\pi}[\cdot]$ is the value of a variable given the policy $\pi$; $t$ is a time step, and $P_t, A_t$ denotes the selected pick-up area and action at time $t$, respectively;
$S_{P_{i+j+1}}$ is the reward. 
The optimal state-value function is defined as:
\begin{equation}
\begin{split}
    & q_*(P_i,a)=E\left[S_{P_{i+1}}+\gamma \max_{\pi}v_{\pi}(P_{i+1})| P_t = P_{i}, A_t=a \right]\\
     & =\sum_{S_{P_{i'}},P_{i'}}p(P_{i'}, S_{P_{i'}}|P_{i},a)\left[S_{P_{i'}}+\gamma \max_{a'}q_*(P_{i'}, a')\right],\\
\end{split}
\end{equation}
where $v_{\pi}(P_{i+1})$ denotes the value function of a state $P_{i+1}$ under a policy $\pi$  and $p(P_{i'}, S_{P_{i'}}|P_{i},a)$ is the action-reward probability.  

Considering that there are usually many potential pick-up areas in a city (e.g., NYC), it would be expensive to obtain the exact policy value.
Temporal Difference (TD) learning is designed to explore the policy under an incomplete environment \citep{tesauro1992practical}, which is a good fit for our task.

The estimation of the state value under the policy $\pi$ via TD learning is shown as follows.
\begin{equation}
    v(P_i) \leftarrow v(P_i)+\eta\left(S_{P_{i+1}}+\gamma v(P_{i+1})-v(P_i)\right),
\end{equation}
where $\eta$ is the learning rate.
Based on the state value, we can evaluate the quality of the given policy $\pi$.
To find the optimal policy, we have to evaluate numerous policies, which is expensive.
To save computing time, we apply the off-policy TD control \citep{watkins1992q} which is defined as:
\begin{equation}
    Q(P_i,A_i) \leftarrow Q(P_i,A_i)+\eta\left(S_{P_{i+1}}+\gamma max_a Q(P_{i+1},a)-Q(P_i,A_i)\right),
\end{equation}
where $A_i$ denotes the action at location $P_i$.
This method is independent of the policy, which helps save time for generating the policy.
It can also directly approximate the optimal action-value function $q^*$ and generate the optimal policy $\pi^*$.

Based on the TD learning method, we have added a self-check mechanism to accelerate the convergence rate.
We have also designed a deep neural network to monitor the situation for the area and to determine the adaptive parameter.
The technical details on the self-check mechanism and the adaptive parameter $F$ will be explained in Section~\ref{sec:method}.

\subsection{Definition of Reward}
The reward is an essential component in RL and will guide the solution exploration.
Given that the profitability is usually considered as an essential evaluation metric for route recommendation \citep{qu2014cost}, we define the reward $S_{P_i}$ as the estimated earning at pick-up area $P_i$.
Following \citep{dong2014recommend}, we obtain the minute-level earning for each pick-up area at a given time based on historical data.
The drop-off point is predicted by capturing the distribution from historical data.
After we obtain the potential drop-off areas, we can estimate the earning of a given pick-up area $P_i$ at a specific time, as long as we compute the delivery time of <$P_i$, $D_i$>.
By checking the historical records of <$P_i$, $D_i$>, the delivery time is not difficult to be estimated (e.g., the average traveling time).

For the case that <$P_i$, $D_i$> has no matched historical record but connected as defined in \citep{cayula1992edge}, the delivery time is calculated as the weighted shortest path.
The weights of the directly connected area are equal to the average delivery time in the historical records.

If neither <$P_i$, $D_i$> appears in historical data nor connected, we determine the delivery time by randomly picking based on corresponding distribution.

We assume the delivery time in the data follows a normal distribution with the mean delivery time $\mu_{del}$ and the standard deviation $\sigma_{del}$.
Then, for a more effective sampling process, we set a lower bound $LB_{del}=\mu_{del}-3\sigma_{del}$, and an upper bound $UB_{del}=\mu_{del}+3\sigma_{del}$. 
We find that this range covers $99.7 \%$ cases in our data.
Thus, the delivery time $t_{del}(P_i , D_i)$ can be estimated as:
\begin{equation}
t_{del}(P_i , D_i)=
\begin{cases}
average~delivery~time, &\text{<}P_i, D_i\text{>} \text{exists~or~connected} \\
\sim U\left(max(0,LB_{del}),UB_{del}\right), & \text{otherwise}
\end{cases}
\end{equation}

Similarly, the cruising time from the drop-off point $D_i$ to the next pick-up point $P_{i+1}$ can be defined based on the lower and upper bounds, $LB_{cru}=\mu_{cru}-3\sigma_{cru}$ and $UB_{cru}=\mu_{cru}+3\sigma_{cru}$ as the following:
\begin{equation}\label{eqn:cru}
t_{cru}(D_i , P_{i+1})=
\begin{cases}
average~delivery~time, &\text{<}D_i, P_{i+1}\text{>} \text{exists~or~connected} \\
 \sim U(max(0,LB_{cru}),UB_{cru}), & \text{otherwise}
\end{cases}
\end{equation}

\textbf{Reward Function}. Considering the route profitability, the reward can be defined as:
\begin{equation}
    reward (P_i)=\frac{INC(P_i) \cdot t_{del}(P_i , D_i)}{t_{del}(P_i , D_i)+t_{cru}( D_i,P_{i+1} )},
\end{equation}
where $D_i$ is the drop-off point for $P_i$; $P_{i+1}$ is the next pick-up point.

When estimating the income, we clean the row data to adjust abnormal pick-up points (e.g., data errors or low-probability cases).
We first calculate the average income (per minute) for the time period.
Given the delivery time for a pick-up point, the target income is defined as the multiplication of the average income of the pick-up point and the delivery time.
If the real total income for the given area is higher than the target income, it will be replaced by the target income.
This process removes the outliers based on each pick-up point's local information.

Then, we further modify the data by considering the global reasonableness. The following data cleaning process is performed:
\begin{equation}
INC(P_i) = 
\begin{cases}
\mu_{INC}+3\sigma_{INC}, & \text{if}~INC(P_i)>\mu_{INC}+3\sigma_{INC}\\
\lambda \cdot INC(P_i), & \text{if}~\#~of~pickups~at~P_i<average~\#~of~pickups
\end{cases}
\end{equation}
We set $\lambda=0.5$ in the experiments. 
This process will handle pick-up points with abnormally high-income and those with insufficient historical records.
Note that the reward function can be modified based on different evaluation metrics.

\subsection{Definition of Action-Reward Probability}
Given a dynamic real-world scenario, the action-reward probability is used to evaluate the probability that a driver will receive the reward from the drop-off area to the next pick-up area.
As mentioned in \citep{veloso2011urban}, taxi drivers usually do not want to travel a long distance for the next pick-up location, which is related to the cruising time.

Since cruising time varies for different locations, we normalize the predicted cruising time, and the probability related to cruising time at area $P_{i+1}$ is defined as:
\begin{equation}\label{eqn:crupro}
    p_{cru}(D_i,P_{i+1})=
    \begin{cases}
    \beta, & \text{<$D_i, P_{i+1}$> not~connected}\\
    1-\beta\frac{\left|t_{cru}(D_i,P_{i+1})-\min_{j} t_{cru}(D_j,P_{j+1})\right|}{\max_{j} t_{cru}(D_j,P_{j+1})- \min_{j } t_{cru}{(D_j,P_{j+1}})}, & \text{otherwise}
    \end{cases}
\end{equation}
The historical data of pick-up information has an influence on the pick-up probability \citep{dong2014recommend}.
As mentioned in \citep{rong2016rich}, the pick-up probability can be represented by the proportion of successful pick-ups.
We thus count the records of successful pick-ups from the historical data and normalize them.
The probability related to frequency for pick-up area $P_{i+1}$ is defined as:
\begin{equation}\label{eqn:pickuppro}
    p_{pick}(P_{i+1})=\frac{\#~of~pickups~at~P_{i+1}-minimum~\#~of~pickups}{maximum~\#~of~pickups-minimum~\#~of~pickups}.
\end{equation}
Based on Eqs. \eqref{eqn:crupro} and \eqref{eqn:pickuppro}, the probability of getting reward at the pick-up point $P_{i+1}$ can be represented as:
\begin{equation}\label{eqn:getreward}
p_{gr}(D_i , P_{i+1})=\alpha_1 p_{cru}(D_i , P_{i+1})+ \alpha_2p_{pick}(P_{i+1}),
\end{equation}
where $\alpha_1+\alpha_2=1$ and $D_i$ is the drop-off area.
As suggested in \citep{lu2016intelligent}, we add a parameter $\omega$ to $p_{gr}$ to represent the effect of anomalies (e.g., unexpected social events) on the pick-up probability.
Then Eq. \eqref{eqn:getreward} can be rewritten as:
\begin{equation}\label{eqn:getreward2}
p_{gr}(D_i , P_{i+1})=\omega\left(\alpha_1 p_{cru}(D_i , P_{i+1})+ \alpha_2p_{pick}(P_{i+1})\right).
\end{equation}
We set $\omega=1$ as default to represent the normal case; however, if the pick-up times at an area is less than $80\%$ of the previous year, then we view this area as an abnormal area and set $\omega<1$.

\textbf{Action-reward Probability Function}. 
Based on Eq. \eqref{eqn:getreward2}, the action-reward probability for the pick-up point $P_{i+1}$ can be written as:
\begin{equation}
    p(P_{i+1},S_{P_{i+1}}|P_{i},a)=
    \begin{cases}
    p_{gr}(D_i , P_{i+1}), & S_{P_{i+1}}=reward(P_{i+1})\\
    1-p_{gr}(D_i , P_{i+1}),& S_{P_{i+1}}=0\\
    \end{cases}
\end{equation}
\begin{algorithm}[h]
\SetAlgoLined
\SetKwInOut{Input}{Input}
 \Input{$r$, $t$}
 $t_{sel}, count, profit \leftarrow 0$\;
 \While{$t_{sel}<t$}{
  $profit \leftarrow profit+E\left(earning~for~r[count]\right)$\;
  $t_{sel} \leftarrow t_{sel}+delivery~time~for~r[count]+cruising~time~for~r[count]$\;
  $count \leftarrow count+1$\;
 }
 Output profit\;
 \caption{Main Function for Path Evaluation: eval}
 \label{alg:eval}
\end{algorithm}

\section{Enhanced Temporal Difference Learning}\label{sec:method}
This section discusses our optimization method designed for the reinforcement learning framework with related theoretical analysis.

\subsection{Optimization}
It is computationally expensive to acquire the full knowledge of the environment for a dynamic system.
As a model free reinforcement learning (RL) method, temporal difference (TD) learning holds the strength in exploring dynamic environment with unknown outcomes. 
While traditional TD learning searches a sequence of optimal solutions for all the states, we propose to focus on the states within a given time interval to avoid unnecessary computing. 

To accelerate the convergence rate and improve the outcome, we introduce a self-checking mechanism and an adaptive parameter into the TD learning.
During the exploring process, if steps are divisible by a predetermined integer $\tau$, the exploring process stops and the current policy will be checked.
If the current policy is better than the preserved policy, it replaces the preserved policy and continues to update the current one; the model will also set the count variable $\Gamma$ to zero.
Otherwise, $\Gamma$ increases by one and compared to the pre-determined parameter $F_c$.
If $\Gamma < F_c$, the update will continue on the current policy.
If $\Gamma = F_c$, the model will restart and explore based on the original blank policy.

\subsubsection{Self-check Mechanism} 
The self-checking process is equivalent to a Markov Chain Process in which the transition rate is controlled by an adaptive factor called \textit{failure count}.
If the count variable $\Gamma$ is smaller than the \textit{failure count}, the transition probability is equal to zero; otherwise, it is equal to one. 

After $\tau$ iterations, the model will evaluate the current optimal route and compare it to the preserved one.
If the current recommendation is better than the preserved one, then the current recommendation will replace the preserved one, and the $\Gamma$ will return to zero.
The transition probability is also equal to zero.
Otherwise, $\Gamma$ will increase by one, and the current policy update will continue. 

Upon $\Gamma$ being equal to the \textit{failure count} $F_c$, it indicates that the current result does not improve for continuous $F_c$ times of check.
In this case, we set the transition probability to one.
The model will preserve the current recommendation and restart the exploration from the original policy.
Importantly, the effectiveness of the self-check mechanism can be theoretically verified.
\begin{proposition}\label{th1}
Given that $k$ is the number of restarts during the self-check process, if $k>0$, the self-check mechanism enlarges the possibility to locate the optimal solution in the dynamic route optimization problem. 
\end{proposition}
\begin{proof}
Suppose that the probability for choosing the right state for TD learning and our method ATDSC are denoted as $p_{td}$ and $p_{atdsc}$. 
For proving the proposition, we have to show that $p_{atdsc} > p_{td}$ when $k>0$. 
In a dynamic route recommendation system, suppose that we want to find the next optimal pick-up point and and all the states are equally distributed; we set $\tau=1$.
Given $M$ states, we assume that there are $M$ iterations.
As $M$ is large, we suppose that the initial state will only be visited once.
For the TD learning, since the initial state will only be visited once, and all the states are equally distributed, the probability for choosing the right state of the next pick-up point is:
\begin{equation}\label{PTD}
    p_{td}=\frac{1}{M}.
\end{equation}

For the self-checking enhanced TD learning method, we assume that there exists a $F_c$ that will lead to $k$ times of model restarting.
If $k=0$, which indicates that the model will not restart and the result for checking is improved, the probability for choosing the right state is the same as the classical TD learning method.
The probability for such cases is the lower limit for our method.
There also exists an extreme case that the model will always restart after $F_c\tau$ steps. 
As $\tau=1$, then $k=\lfloor \frac{M}{F_c} \rfloor$, 
and the probability for such case should be the upper limit for our method.
The probability for our model ATDSC to select the right state is shown as:
\begin{equation}\label{PTDSC}
p_{atdsc} = 1-\prod_{i=0}^{k} \left(1-\frac{1}{M}\right),
\end{equation}
where $0 \leq k \leq \lfloor \frac{M}{F_c} \rfloor$.
Based on the following fact:
\begin{equation}
    \begin{cases}
    \frac{1}{M}=1-\left(1-\frac{1}{M}\right) \\
    (1-\frac{1}{M})\prod_{i=1}^{k} \left(1-\frac{1}{M}\right)<1-\frac{1}{M}
    \end{cases}
\end{equation}
we can conclude that $\frac{1}{M} \leq 1-\prod_{i=0}^{k} \left(1-\frac{1}{M}\right)$, indicating that as long as $k>0$, the self-checking mechanism has a higher possibility to achieve the optimal state.
\end{proof}
By adjusting $F_c$, we can ensure that the model will restart and $k>0$.
Proposition \autoref{th1} suggests that our model with the proposed self-check mechanism should achieve a better performance than the original version of TD learning.
\begin{algorithm}[h]
\SetAlgoLined
\SetKwInOut{Input}{Input}
\Input{time limit $t$, initial state $s_0$, failure rate $\Gamma$, self-check iterations $\tau$, learning rate $\eta$, discount rate $\gamma$, number of states $M$ and restart integer $c$}
 $ r^*,Q^*,Q \leftarrow 0$\;
 $s \leftarrow s_0$\;
 $I_o\leftarrow ANN(\text{area information})$\;
 \eIf{$I_o$==1}{
 $F_c\leftarrow c\cdot(\frac{N_{normal}}{M})^{1/3}$\;
 }
 {
 $F_c\leftarrow c$\;
 }
 \While{The stopping criteria is not met}{
  Randomly choose an action $a$ from $s$\;
    Taking action $a$, get $s'$ and $Rew(s,s')$\;
    $Q(s,a) \leftarrow Q(s,a)~+$
    $\eta\left(Rew(s,s') +\gamma max_{a'} Q(s',a')-Q(s,a)\right)$\; 
    $s \leftarrow s'$\; 
 \If{iteration divides $\tau$ }{
 Generate path $r$ based on $Q$ and $S_0$\;
 \eIf{\text{eval}($r,t$) $-$ \text{eval}($r^*,t$) > 0}{
 $r^*\leftarrow r$\;
 $ Q^* \leftarrow Q$\;
 $ \Gamma \leftarrow 0$\;
 }{
 \eIf{$\Gamma < F_c$}
 {
 $\Gamma \leftarrow \Gamma+1$\;
 }{
 
 $Q,\Gamma \leftarrow 0$\;
 $s \leftarrow s_0$\;
 }
 }
 }
 }
 \caption{ATDSC}
 \label{alg:TDSC}
\end{algorithm}

\subsubsection{Adaptive failure count $F_c$}
The failure count $F_c$ not only helps speed up the convergence rate, but also controls the performance of the model.
If the failure count is too large, the performance will be improved too slowly because the model has to wait for $F_c\tau$ steps to restart.
If the failure count is too small, then the exploring process of the TD learning will be restricted as the model always restarts without getting enough knowledge of the current environment.
To properly select $F_c$, we formalize the following proposition.

\begin{proposition}\label{th2}
As long as the reward is constrained within a certain scope, the difference between action-reward probabilities is inversely proportional to the failure count $F_c$. 
\end{proposition}
\begin{proof}
The difference of the expected reward $E(P_i)$ and $E(P_j)$ can be represented as:
\begin{equation}\label{dif1}
    |E(P_i)-E(P_j)|=\left|p_{gr}(D_{i-1},P_i)reward(P_i)-p_{gr}(D_{j-1},P_j)reward(P_j)\right|,
\end{equation}
where $D_{i-1}$ is the drop-off point before $P_i$, and $D_{j-1}$ is the drop-off point before $P_j$.
Without loss of generality, we assume that $reward(P_i)>reward(P_j)>0$. 
Then, we can obtain the following inequality:
\begin{equation}\label{dif2}
    |E(P_i)-E(P_j)|>\left|reward(P_j)\left(p_{gr}(D_{i-1},P_i)-p_{gr}(D_{j-1},P_j)\right)\right|.
\end{equation}
Since we normalize the reward based on the restriction, the value of $reward(P_j)$ is controllable.
If the difference of action-reward probabilities is large, $|E(P_i)-E(P_j)|$ should be larger.
The policy in TD learning is improved by expected rewards. 
A larger difference indicates that the model needs to do less to explore further. 
Thus, the failure count can be smaller.
On the other hand, considering that a small $|E(P_i)-E(P_j)|$ suggests a smaller difference between state values, the model needs to explore more steps for a better decision.
\end{proof}

Proposition \autoref{th2} offers the rule for determining the value of $F_c$. 
Since we add a penalty term to Eq. \eqref{eqn:getreward2}, the difference between expected rewards $|E(P_i)-E(P_j)|$ must be larger in abnormal areas.
Given the strength of deep learning models in addressing classification problems \citep{guo2020weighted,guo2019weighted,an2020rahm,liu2020unified,sun2019exploiting,zhang2018caden}, we apply a deep Artificial Neural Network (ANN) to help decide to change the value of the failure count $F_c$.
The inputs include the current average successful pickups for each area and the average number of successful pickups in the previous year. 
We label the records as follows. 
According to the pandemic outbreak dates reported in news (e.g., \footnote{https://www.prnewswire.com/news-releases/impact-of-covid-19-on-the-taxi-and-limousine-services-market--tbrc-report-insights-301054745.html}), if the number of abnormal area is more than half of area numbers, we set the label to one, otherwise, zero.
The output of the ANN $I_o$ can be defined as:
\begin{equation}
    I_o=
    \begin{cases}
    0, & \text{unchanged $F_c$} \\
    1, & \text{changed $F_c$}
    \end{cases}
\end{equation}

Proposition \autoref{th2} also suggests the following pattern regarding the \textit{failure count}.

\begin{lemma}\label{lem1}
As long as normal areas exist, the more the abnormal areas we have, the smaller the failure count will be.
\end{lemma}
\begin{proof}
Proposition \autoref{th2} suggests that a large difference of expected rewards can help distinguish between the states.
Based on the condition that normal areas exist, the difference between the normal area and the abnormal area must be large, reflected by the effect of the penalty term.
Since we only care about the areas within a given traveling time range, and the expected rewards for the normal areas are assumed to be larger than those of abnormal areas, we ignore the expected reward of abnormal area while guaranteeing to explore all normal areas.
Thus, $F_c$ should be proportional to the number of normal areas and negatively proportional to the number of abnormal areas.

\end{proof}
Lemma \autoref{lem1} shows that $F_c$ is related to the number of normal areas, and hence we can define the failure count function as follows.

\textbf{Failure Count Function.}
The function for the failure count $F_c$ is defined as:
\begin{equation}
    F_c=
    \begin{cases}
    c, &I_o=0\\
    c\cdot(\frac{N_{normal}}{M})^{1/3}, &I_o=1
    \end{cases}
\end{equation}
where $N_{normal}$ is the number of normal area; and $c$ is the default value (integer).
\subsection{Adaptive TD Learning with Self-Check}
Algorithm~\ref{alg:eval} describes how we evaluate the policy.
Given the path $r$ under the policy and time period, the algorithm will calculate the expected earning on the path within the given time period.

Note that, although our method is designed for sequential route recommendation (long-term), it can also be implemented for a one-step recommendation task. 
To do so, we can set a small number for the expected working hours (e.g., half-hour).


Recall that the expected income is calculated by multiplying the delivery time by the earning rate from historical data.
Once the earning is added, the delivery time to the current drop-off area and the cruising time from the current drop-off point to the next pick-up area will all be accumulated.
If the accumulated time exceeds the given time variable, the algorithm will stop exploring and output the total profit.

The detailed procedure about the update strategy for our method is illustrated in Algorithm~\ref{alg:TDSC}.
The predicted time range $t$ and the initial state $s_0$ are user-determined parameters. 
The model will adjust $F_c$ based on the result from deep learning initially.
Note that during the TD learning, when the number of iterations reaches the multiple of $\tau$, the model will hold the search until it finishes comparing the current policy with the preserved one. 

During the comparison, the evaluation is done by Algorithm~\ref{alg:eval}.
If the result is better than the preserved one, the search will continue, and the current one will replace the preserved policy.
Otherwise, the preserved policy stays.
If the result does not improve after $c \tau$ iterations, then the model will launch a new policy exploration.


\section{Experiments}\label{sec:result}
In this section, we discuss the data, experimental settings, and results.

\subsection{Data and Preprocessing}
\textit{Data Description.}
Our experiments are based on the taxi data from New York City \footnote{https://www1.nyc.gov/site/tlc/about/tlc-trip-record-data.page}{\label{note1}}.
The time period of our data ranges from January 2020 to June 2020, covering the periods before, during, and after the first wave of the COVID-19 pandemic outbreak.
Our data contain taxi travel records from five boroughs and a related region of New York City: Bronx, Brooklyn, Manhattan, Queens, Staten Island, and the Newark International Airport (EWR).
The data include fields capturing pick-up and drop-off dates/times, pick-up and drop-off locations, trip distances, total payments, payment types, and driver-reported passenger counts.
A summary of important data statistics is reported in Table~\ref{tab:sumdata}.
Although we have demonstrated our method using the NYC data, our method can be readily applied to other cities as well.
\begin{table}
\centering
\begingroup
\setlength{\tabcolsep}{6pt} 
\renewcommand{\arraystretch}{1} 
  \caption{Summary of Data Statistics.}
  \label{tab:sumdata}
\begin{tabular}{l|cccccc}
\hline
\diagbox{Features}{Month (2020)} & Jan & Feb & Mar & Apr & May &  Jun\\ 
\hline
Total Trip Records ($\times 10^7$) & 11.5 & 11.3 & 5.4  & 0.4 & 0.6 & 1.0 \\
Average Income per Trip (\$) & 19 & 19 & 19 & 16 & 20 & 19 \\
Average Delivery Time (minutes) & 15 & 17 & 15 & 11 & 13 & 14 \\
Average Cruising Time (minutes)& 10 & 9 & 11 & 14 & 15 & 12 \\
\hline
\end{tabular}
\endgroup
\end{table}


\textit{Data Preprocessing.}
We assume that only neighboring areas are connected, otherwise, they cannot be reached directly.
For example, if area A and area C are the neighbors of area B, and area A is not area C's neighbor, then the trip from A to C should pass through B. 
The traveling time from A to C is estimated as the average traveling time from A to B plus the average traveling time from B to C.    
We collect the relations between each area from the Taxi Zone Map. 
For non-reachable area, the estimation of delivery time and cruising time is mentioned in Section~\ref{sec:model}.
Based on \citep{dong2014recommend}, we use the historical records of trip distances, total payment, and pick-up and drop-off locations to compute the average earning rate per minute as well as the pick-up/drop-off frequencies of each location.

\subsection{Experimental Settings}
Now we discuss the parameter settings of our method, the implementation of all benchmark methods, and the validation metrics.
\subsubsection{Parameter Settings}
According to report from the District Department of For-Hire Vehicles, 
the trips decreased about 90$\%$ during the COVID-19 outbreak.
We set $\omega=0.1$ for abnormal areas.
The weights for action-reward probability are set to: $\alpha_1=\alpha_2$.
The non-connected getting reward probability $\beta$ is set to 0.1.
The default value $c$ is set to 8 and the discount rate $\gamma$ is equal to 0.9.
Both the learning rate and the learning rate decay for reinforcement learning are set to 0.01.
The total iterations are set to 300,000.
All algorithms are implemented in Python, and experiments are conducted on the Seawulf, a high-performance computing cluster \footnote{
SeaWulf is a computational cluster in Stony Brook University, using top of the line components from Penguin, DDN, Intel, Nvidia, Mellanox and numerous other technology partners. 
See more information:  https://it.stonybrook.edu/help/kb/understanding-seawulf}. Each processor we used in our experiments has two Intel Xeon E5-2690v3 12 core CPUs and 128 GB DDR4 Memory.
For a fair comparison, we report the mean performance of 30 independent experiments (indicating 30 recommended paths) based on random initial states. 
Note that the error bars in figures represent the standard errors.
\subsubsection{Baselines}
We compare our method (ATDSC) with four baselines, including REI, MPP, MNP, and PCD, in which the REL can be considered the state-of-the-art method of RL-based route recommendation.
\begin{itemize}
    \item \textbf{REI} \citep{ji2020spatio}. 
    The original method focuses on recommending routes under the guidance of the deep RL method.
    The model is led by a classic RL method assuming the full knowledge of the environment is known. Hence it cannot be applied to our problem directly.
    Thus, we implement the RL based on TD learning, with which the full knowledge of the environment is not required.
    \item \textbf{MPP} \citep{yuan2011find}. 
    This method is a greedy method in terms of the pick-up probability. 
    It aims to recommend the area with the maximum pick-up probability to the taxi drivers.
    The pick-up probability is scratched from the historical data.
    \item \textbf{MNP} \citep{qu2014cost}.
    This method aims to maximize the area profit to the taxi driver.
    Since we calculate the minute income for each area, this method is equivalent to the greedy method in terms of the minute income.
    \item \textbf{PCD} \citep{luo2018dynamic}.
    The original method is to recommend the path with minimal potential cruising distance. 
    Since we replace the distance with time, it is equivalent to finding the pick-up point with minimal cruising time in our setting.
\end{itemize}
\subsubsection{Validation Metrics}
The main validation metrics include the hourly and weekly expected income. 
We also compute standard errors to evaluate the reliability of our results based on 30 independent experiments for path exploring.

\textbf{Expected hourly income.}
We simulate the route for one day under different methods, the expected hourly income and standard errors can be computed as:
\begin{equation*}
    \begin{cases}
    & \overline{E}(Hourly_{meth})=\frac{1}{q}\sum_{i=1}^q\frac{E_i(Daily_{meth})}{24}\\
    & \text{standard~error}=\sqrt{\frac{var\left(E_1(Daily_{meth})\cdots E_q(Daily_{meth})\right)}{q}}
    \end{cases}
\end{equation*}
where $\overline{E}(\cdot)$ is the average expectation function; $q$ is the number of path explorations with random initial states; $Daily_{meth}$ represents the total daily income for a driver following method $meth \in$  [REI, MPP, MNP, PCD, ATDSC].

\textbf{Expected weekly income.}
Based on the expected hourly income, we assume that a driver will work ten hours a day and can estimate the weekly income as:
\begin{equation*}
    \overline{E}(weekly_{meth})=\sum_{i=1}^7 10\cdot \overline{E}(Hourly_{meth}^i),
\end{equation*}
where $i=1,2,...,7$ represents the days of a week (from Monday to Sunday), $Hourly_{meth}^i$ is the hourly income; $\overline{E}(weekly_{ATDSC})$ and $\overline{E}(weekly_{basl})$ denote the average expected weekly income using ATDSC and the baseline methods \textit{basl} $\in$ [REI, MPP, MNP,
PCD], respectively.

To facilitate the comparison, we also report the improvement of our method over other baselines.
\begin{equation*}
   \text{Improvement}=ln \frac{\overline{E}(weekly_{ATDSC})-\overline{E}(weekly_{basl})}{\overline{E}(weekly_{basl})}.
\end{equation*}
\begin{figure}
\centering
   \subfigure[January 2020] {\includegraphics[width=0.3\linewidth]{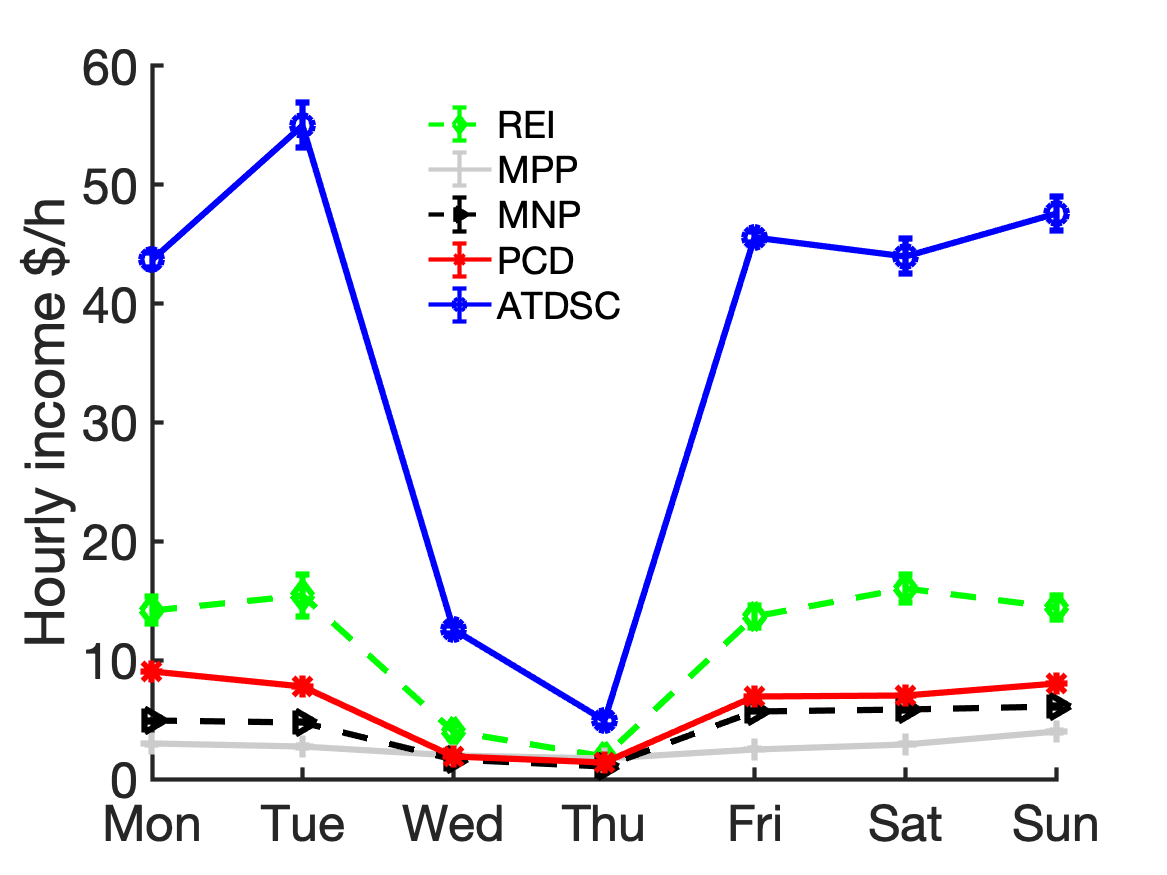}}
   \subfigure[February 2020]{\includegraphics[width=0.3\linewidth]{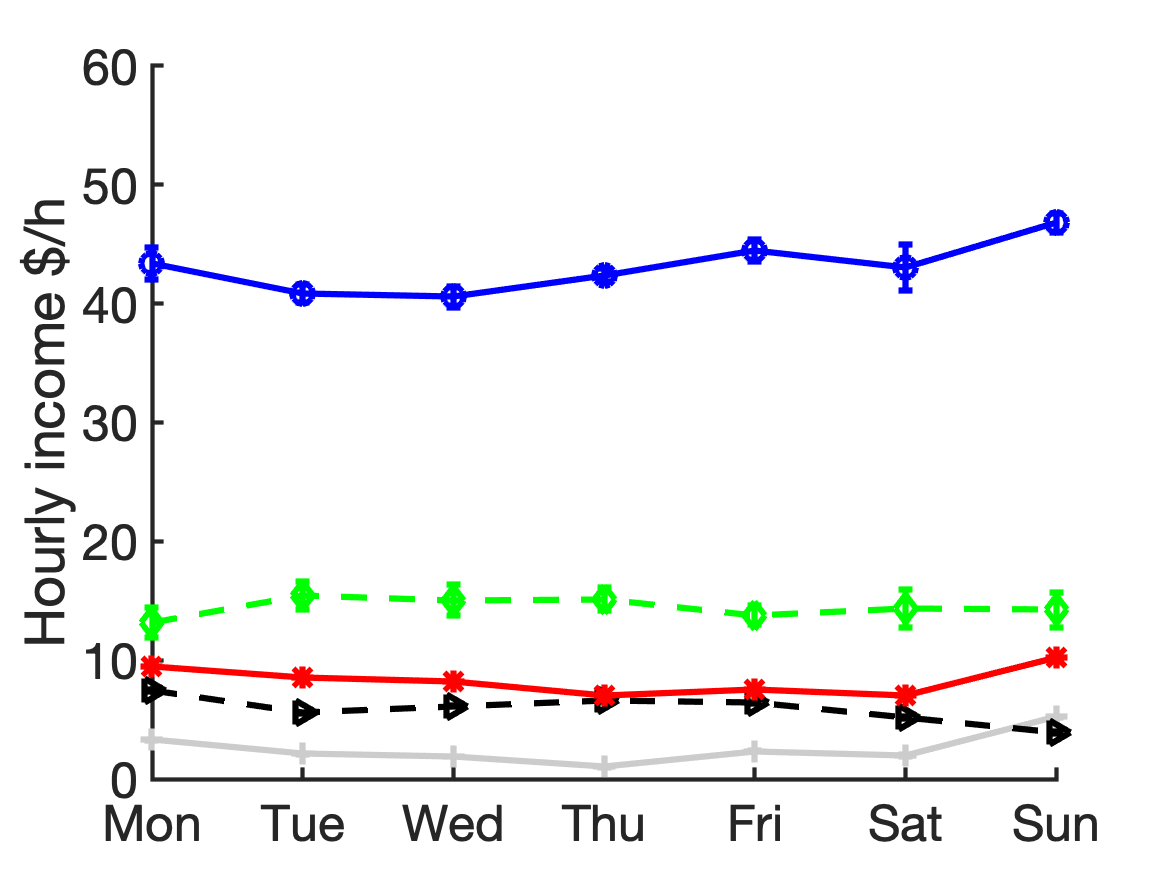}}
   \subfigure[March 2020]{\includegraphics[width=0.3\linewidth]{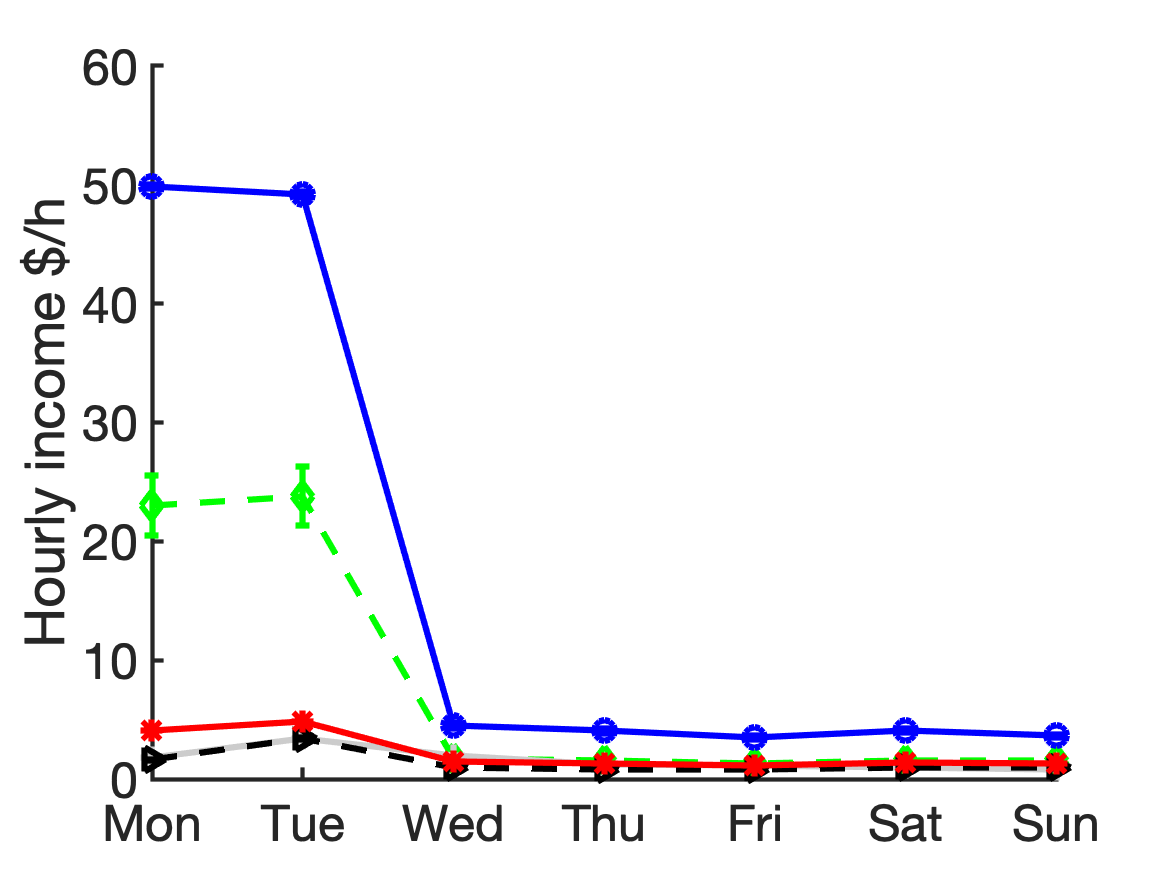}}\\
  \subfigure[April 2020]{\includegraphics[width=0.3\linewidth]{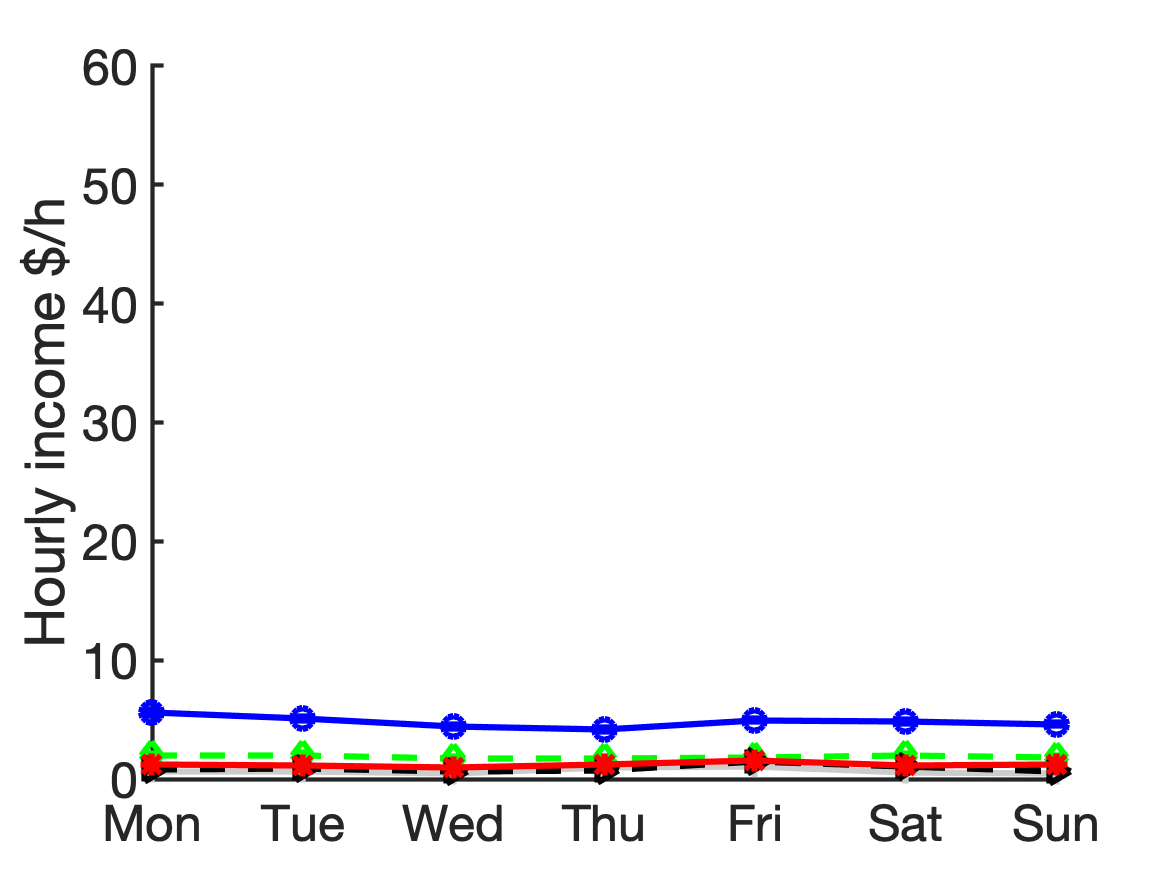}}
   \subfigure[May 2020]{\includegraphics[width=0.3\linewidth]{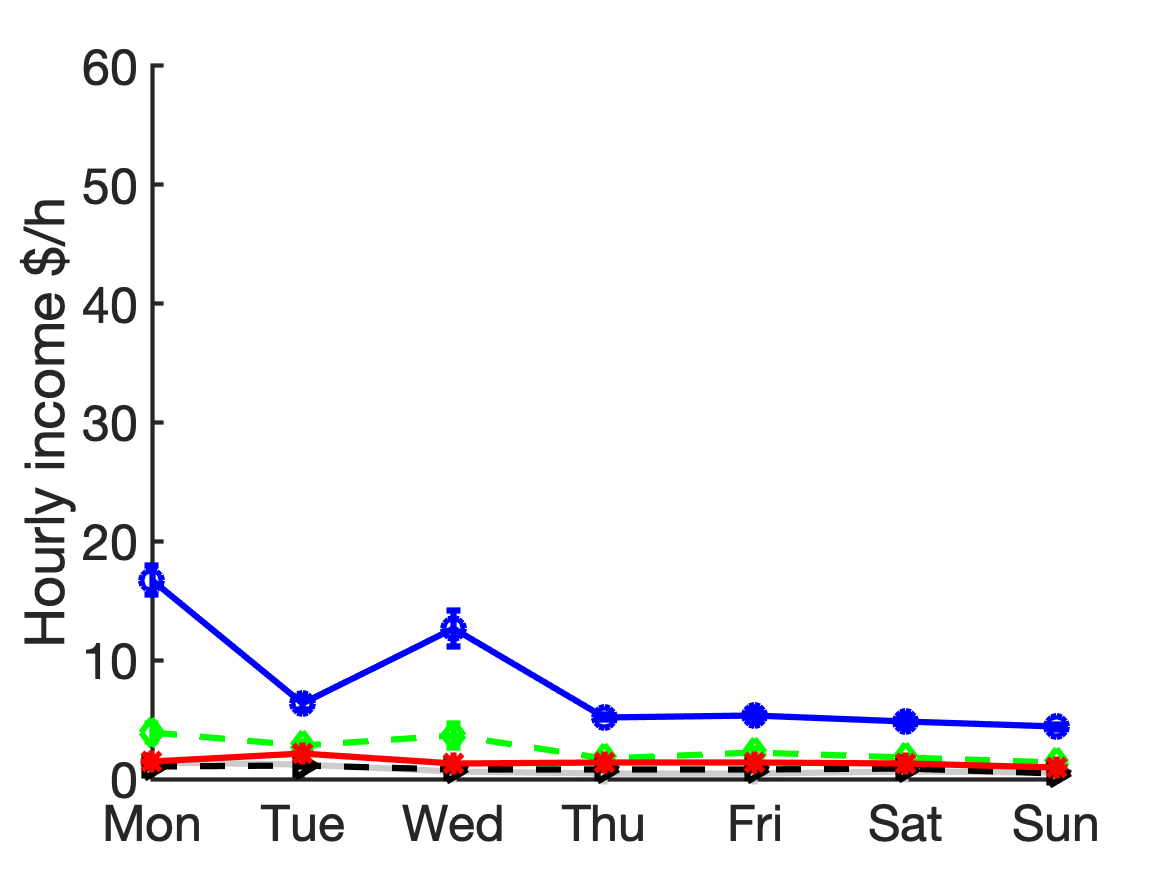}}
  \subfigure[June 2020]{\includegraphics[width=0.3\linewidth]{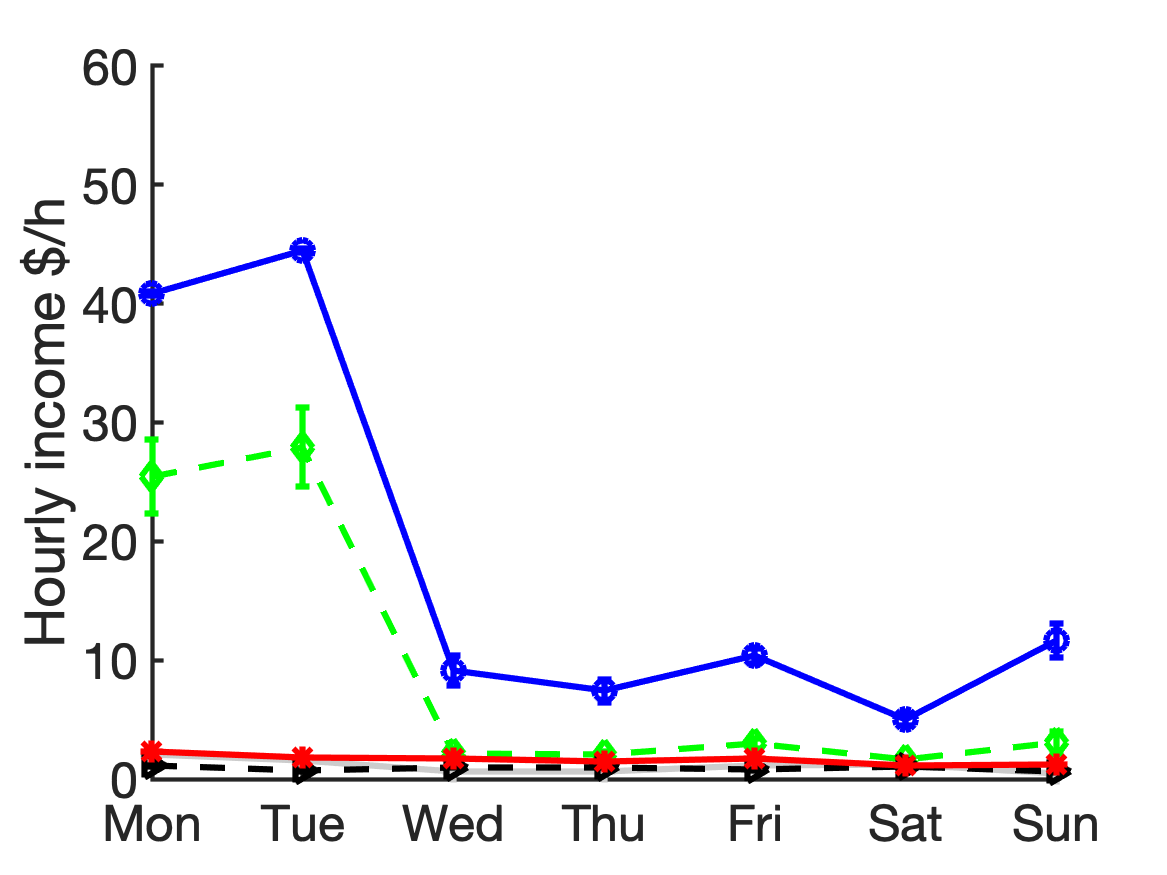}}\\
   \caption{Average Hourly Income.
   }
   \label{fig:hourlyincome}
\end{figure}

\subsection{Results}
To validate our method, we demonstrate the results in the views of average hourly income, weekly income, along with an hour by hour case study.

\subsubsection{Hourly income comparison.}
\autoref{fig:hourlyincome} illustrates the performance comparison of the average hourly income in different months.
Each reported point represents the average hourly income of 30 paths based on random initial states, and error bars represent the corresponding standard errors.
As can be seen, the hourly income was significantly reduced in April, May, and June, comparing to January and February.
This is what we expect given the COVID-19 outbreak in March. 

According to the results, Our method ATDSC shows consistent superiority over other benchmarks, for all weekdays and weekends, before and after the COVID-19 outbreak.
Although there is a significant shrink of income after March for all benchmarks, our method has resulted in a much higher hourly income. 
Especially, in March (when the COVID-19 had just started), our method has done a good job in maintaining the hourly income close to normal rate; and in June (after two months of COVID-19), our method brings the earnings back to the level of March, while all benchmarks remain in low performance.

Since the iterations for \textit{REI} are not large enough to guarantee convergence, it causes the unstable performance in April and May as it nearly achieves the same performance as other baselines.
The difference between our method and \textit{REI} is credited to the efficiency of our self-check mechanism.
\textit{MNP} aims to maximize the area profit with a greedy strategy. Its worse performance indicates that only concentrating on the short term profit will not lead to a long term benefit.
The results of the \textit{MPP} and \textit{PCD} indicate that maximizing the pick-up probability or reducing the cruising time may not guarantee the maximum earnings for drivers.

\begin{figure}
\centering
    \subfigure[January 2020] {\includegraphics[width=0.3\linewidth]{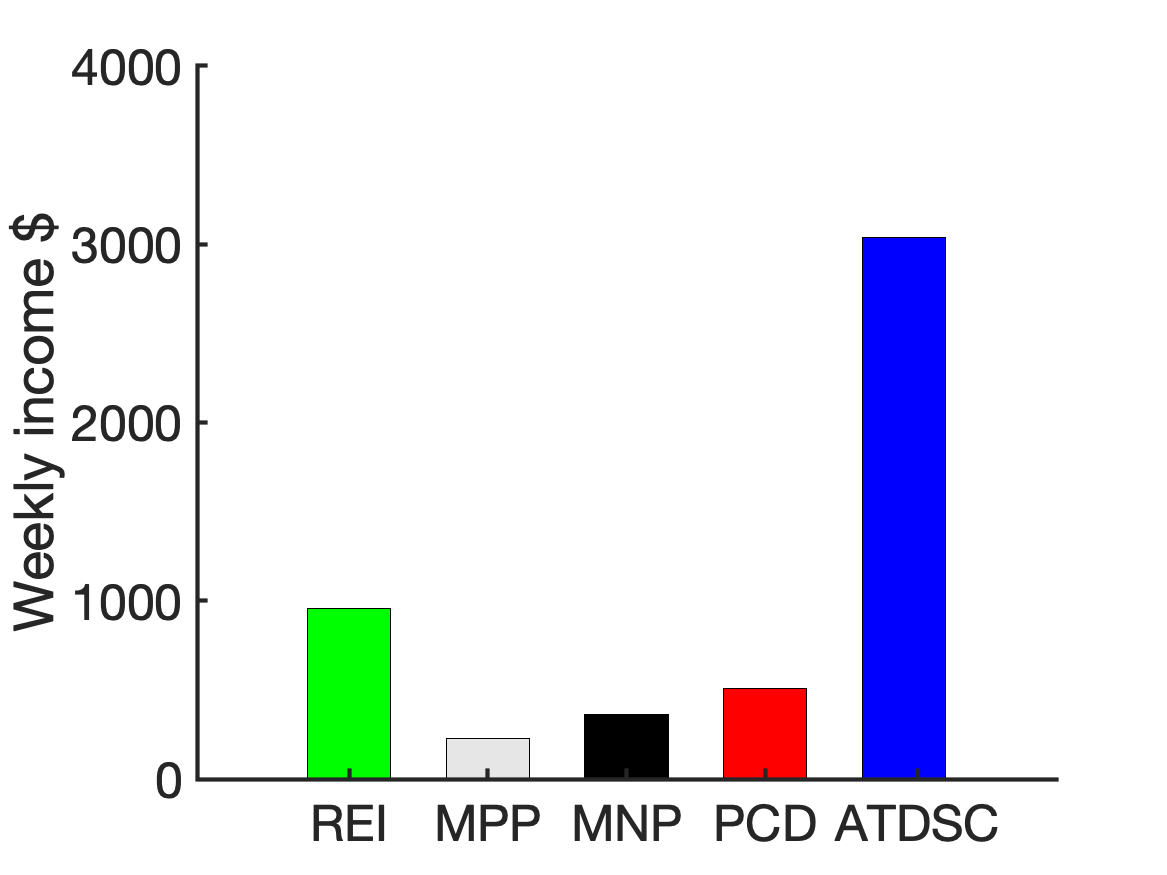}}
  \subfigure[February 2020]{\includegraphics[width=0.3\linewidth]{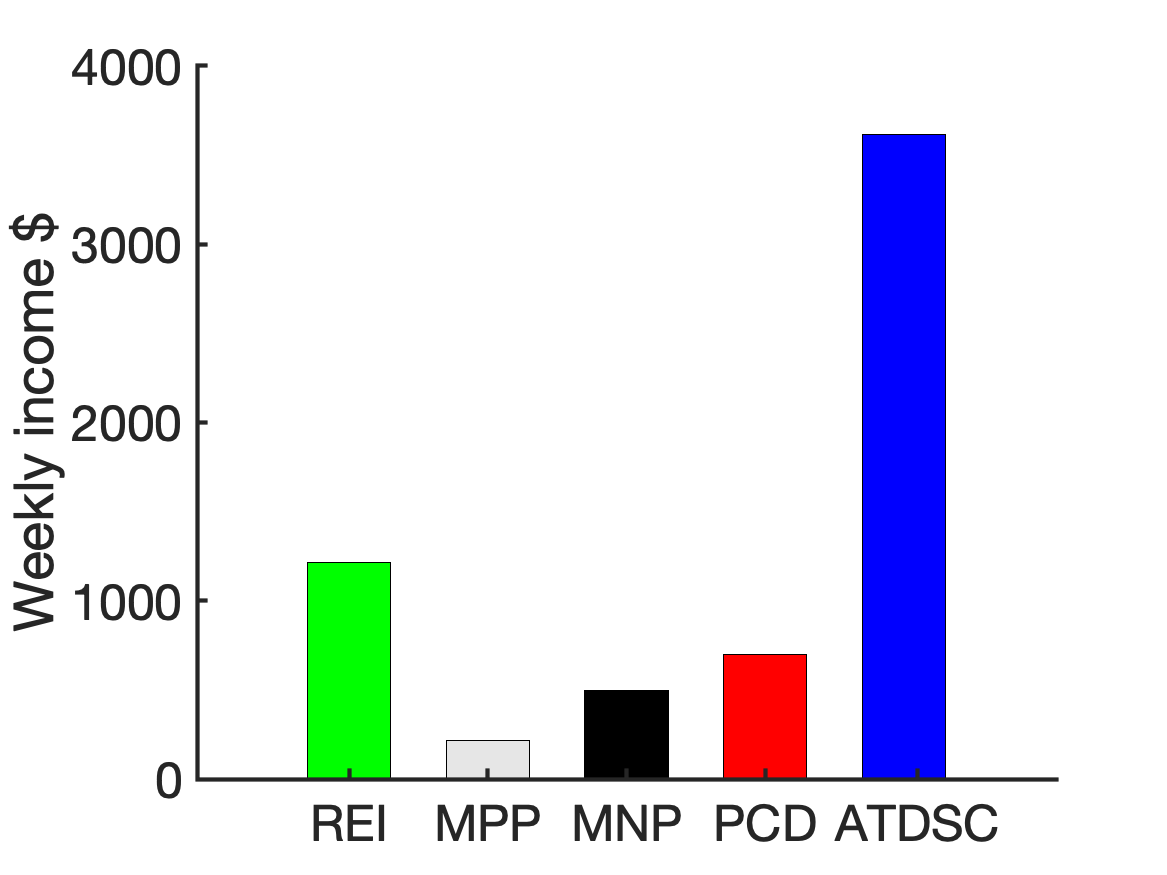}}
  \subfigure[March 2020]{\includegraphics[width=0.3\linewidth]{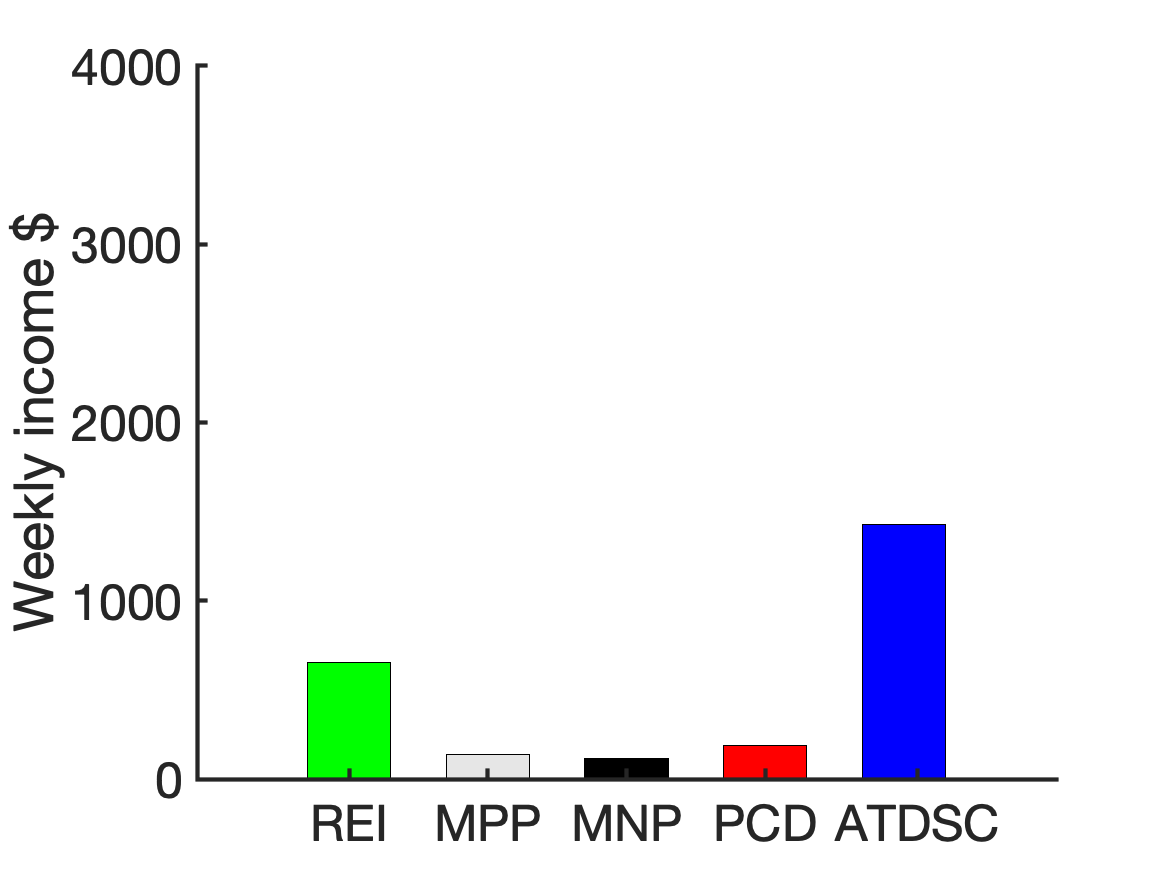}}\\
  \subfigure[April 2020]{\includegraphics[width=0.3\linewidth]{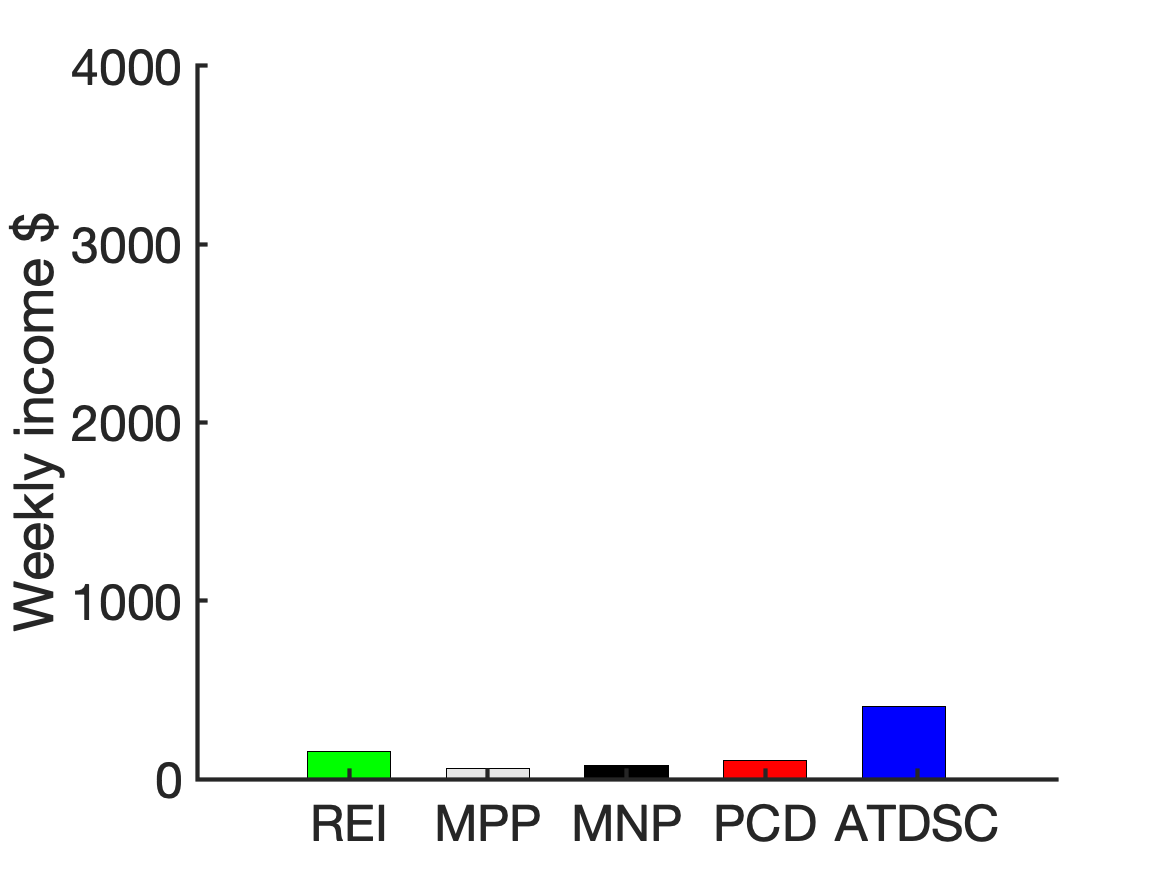}}
  \subfigure[May 2020]{\includegraphics[width=0.3\linewidth]{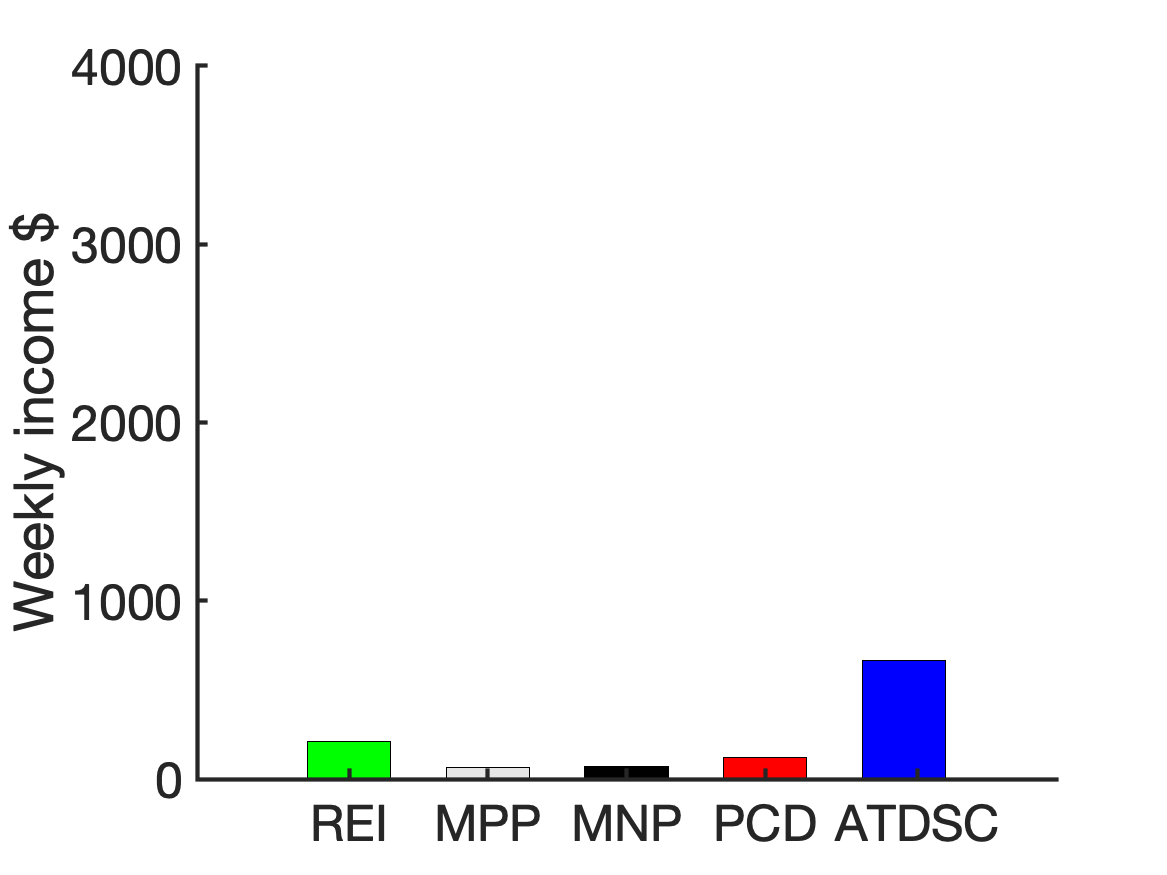}}
  \subfigure[June 2020]{\includegraphics[width=0.3\linewidth]{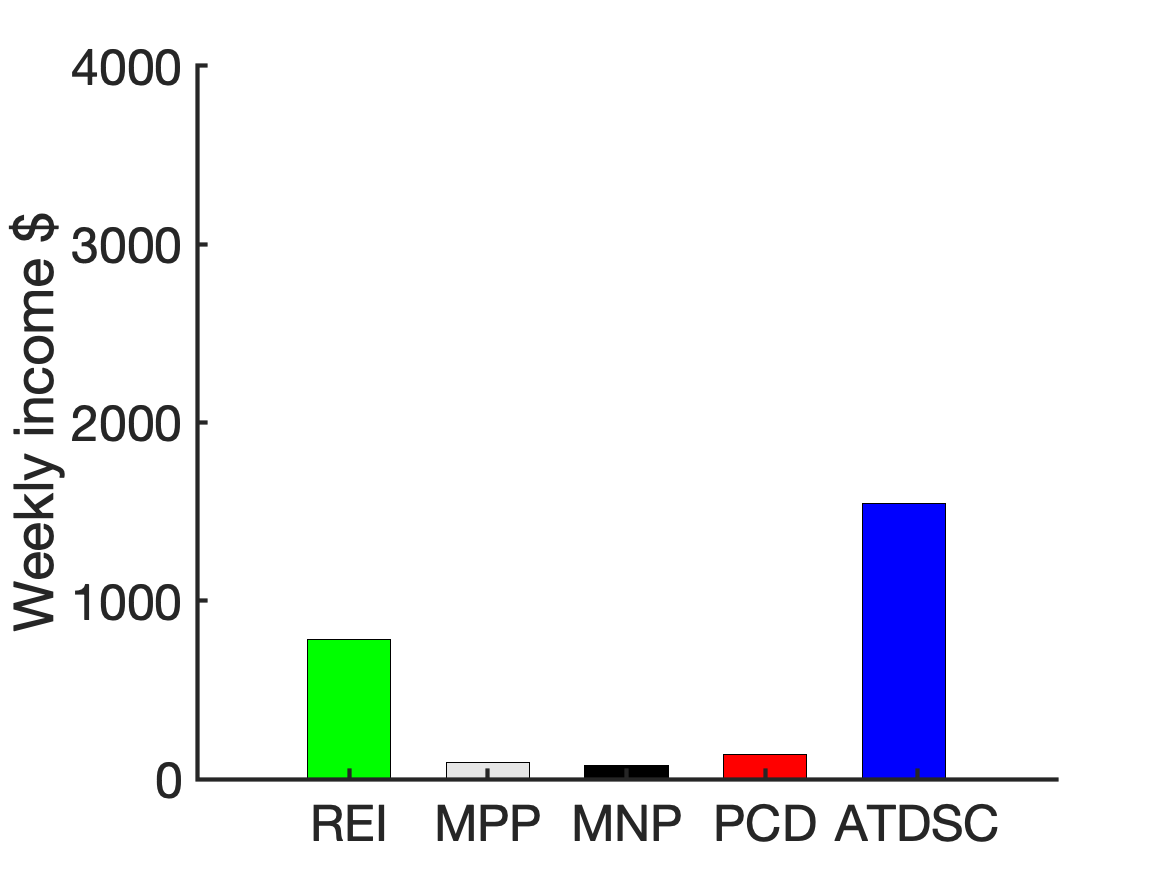}}
   \caption{Average Weekly Income.
   }\label{fig:weeklyimp}
\end{figure}
\begin{table}
\centering
\begingroup
\setlength{\tabcolsep}{6pt} 
\renewcommand{\arraystretch}{1} 
  \caption{Natural Logarithmic Improvement Over the Baselines.}
  \label{tab:imp}
\begin{tabular}{l|cccccc}
\hline
\diagbox{Method}{Month (2020)} & Jan & Feb & Mar & Apr & May &  Jun\\ 
\hline
REI & 0.78 & 0.68 & 0.17 & 0.46 & 0.77 & -0.02 \\
MPP & 2.52 & 2.77 & 2.25 & 1.77 & 2.22 & 2.76 \\
MNP & 2.01 & 1.84 & 2.45 & 1.49 & 2.15 & 2.98 \\
PCD & 1.61 & 1.43 & 1.9 & 1.09 & 1.52 & 2.33 \\
\hline
\end{tabular}
\endgroup
\end{table}

\subsubsection{Average Weekly Income}
Now, we report the average weekly income by assuming that a full-time driver works ten hours a day, seven days a week. 
Related results can be found in \autoref{fig:weeklyimp}.
While the same income reduction pattern can be observed after the COVID-19 outbreak, our method has a consistently better performance than all benchmarks in terms of profitability.
In June, when the estimated incomes are still low based on all benchmarks, our method achieves a similar profit level as in March.
This is partially because our method is capable of checking the abnormal areas dynamically to detect the situation change when people are returning to their normal life. 

The natural logarithmic improvements of our method over other methods are also reported in \autoref{tab:imp}, indicating a stable and consistently superior performance of our method. 
As can be seen that the minimum natural logarithmic improvement in the table is -0.02, indicating a minimum of 98\% improvement.

Note that taxi drivers usually face high costs when providing services. For example, taxi companies usually charge drivers about one-third of their overall gross fare income.
Also, if a passenger pays by credit card, the driver may be charged a minimum 2\% transaction fee.
Some taxi drivers have to pay auto insurance and maintenance by themselves based on their contracts as independent contractors \footnote{https://work.chron.com/much-fare-taxi-drivers-keep-22871.html}. 
Given all these costs, we estimate that a taxi driver's actual profit is approximately 50-60\% of the income. 

\autoref{fig:adaptive} compares the performance between fixed and dynamic failure rate settings, based on our method. 
The results confirm the effectiveness of our adaptive failure rate.
According to Proposition~\ref{th2}, the usage of the adaptive parameter will achieve the convergence faster than the fixed parameter case.
Now it is also supported by the experimental results.
Our method will automatically update the failure rate based on the number of detected normal areas.
The smaller the normal area number is, the smaller the failure rate will be.

\begin{figure}[!] 
\centering
{\includegraphics[width=0.42\linewidth]{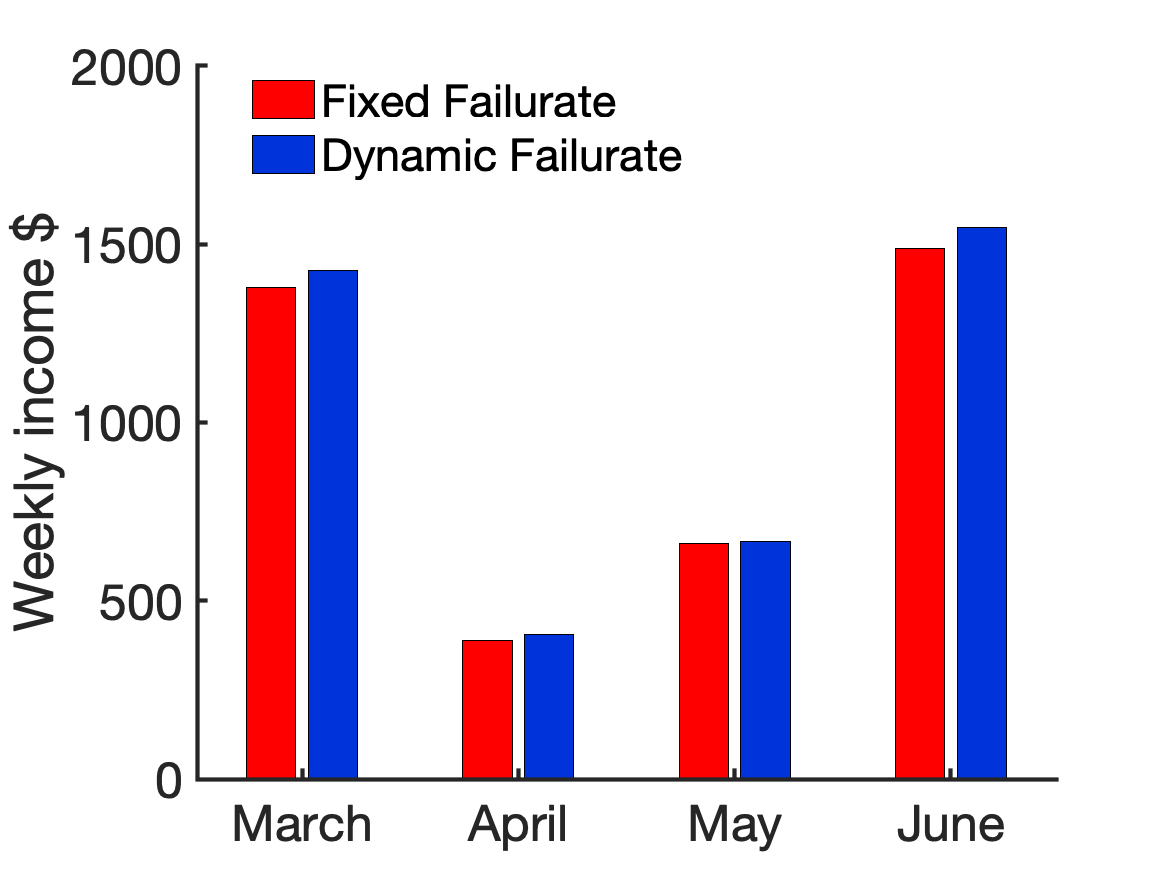}}
   \caption{Weekly Income Comparison (Fixed vs. Dynamic Failure Rate).
   }\label{fig:adaptive}
\end{figure}

\begin{figure}
\centering
   \subfigure[Monday in January 2020] {\includegraphics[width=0.3\linewidth]{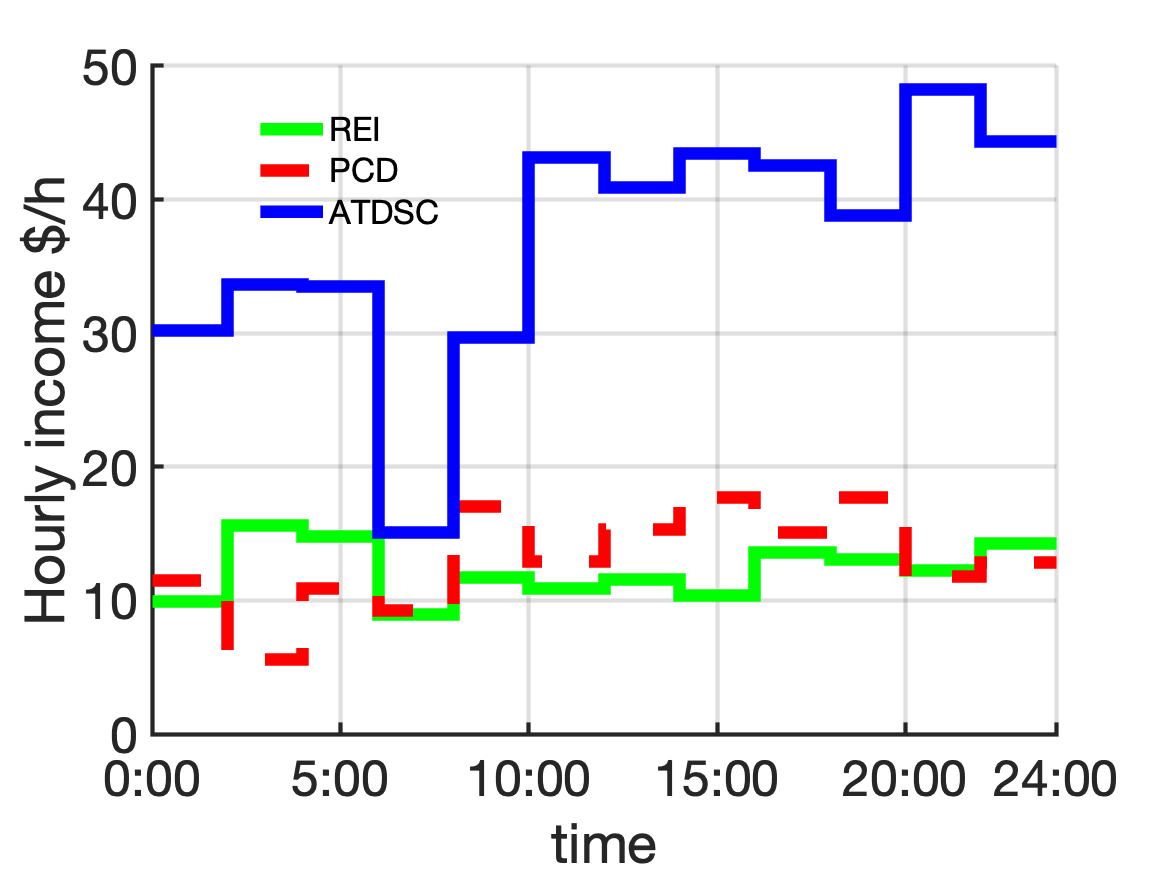}}
   \subfigure[Monday in February 2020]{\includegraphics[width=0.3\linewidth]{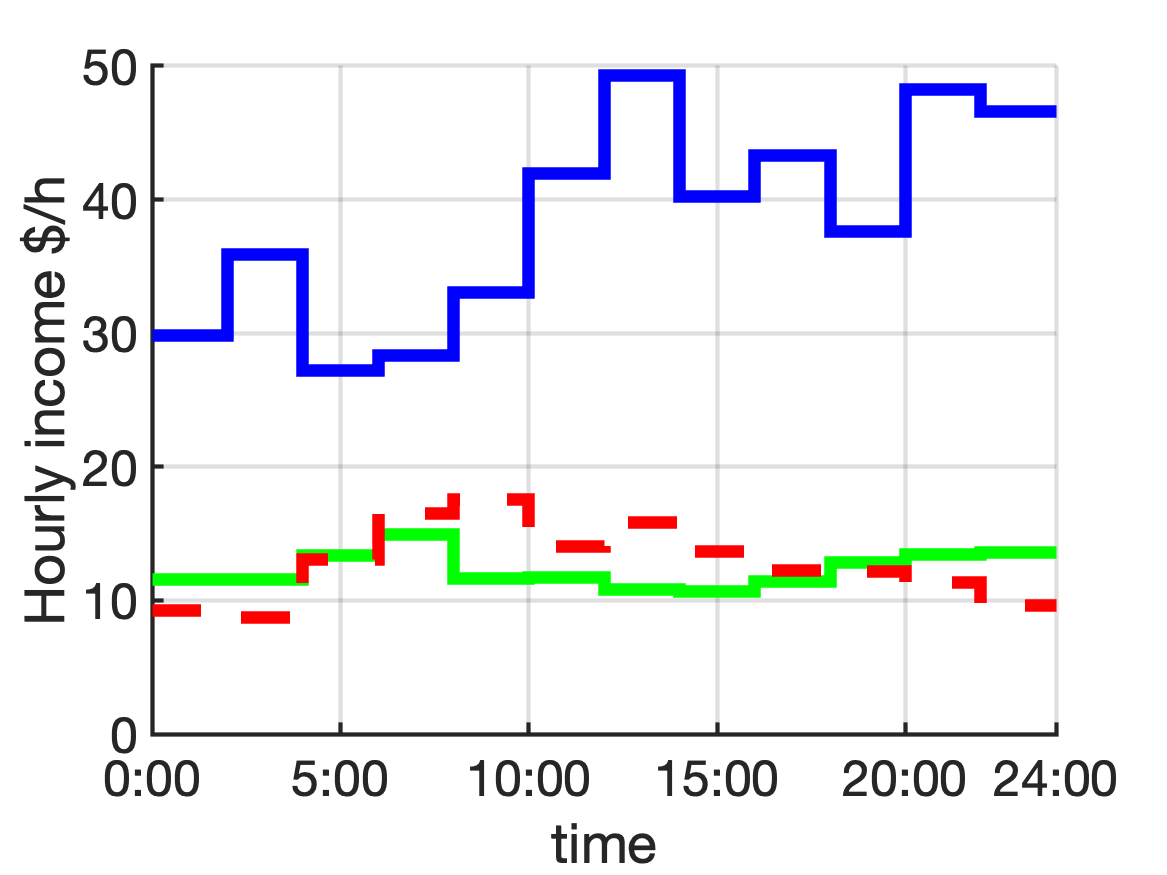}}
   \subfigure[Monday in March 2020]{\includegraphics[width=0.3\linewidth]{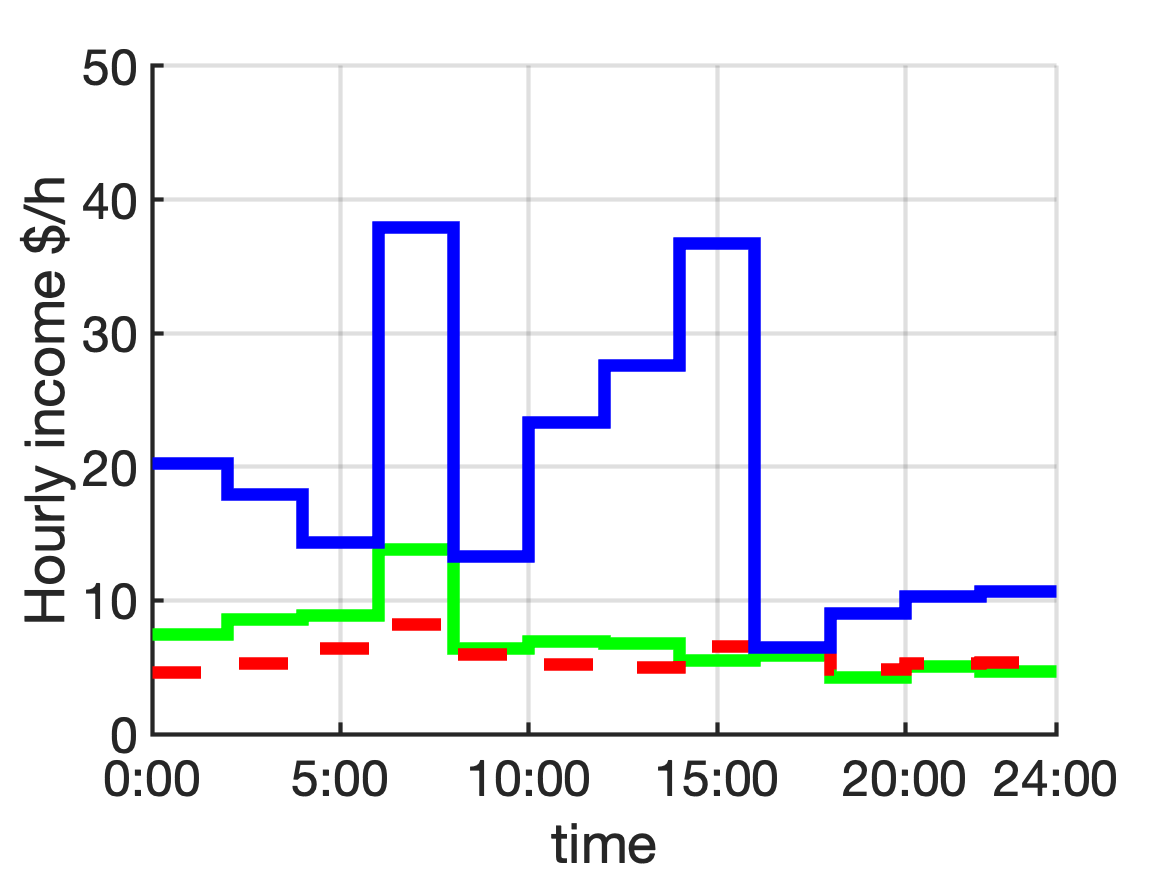}}\\
  \subfigure[Monday in April 2020]{\includegraphics[width=0.3\linewidth]{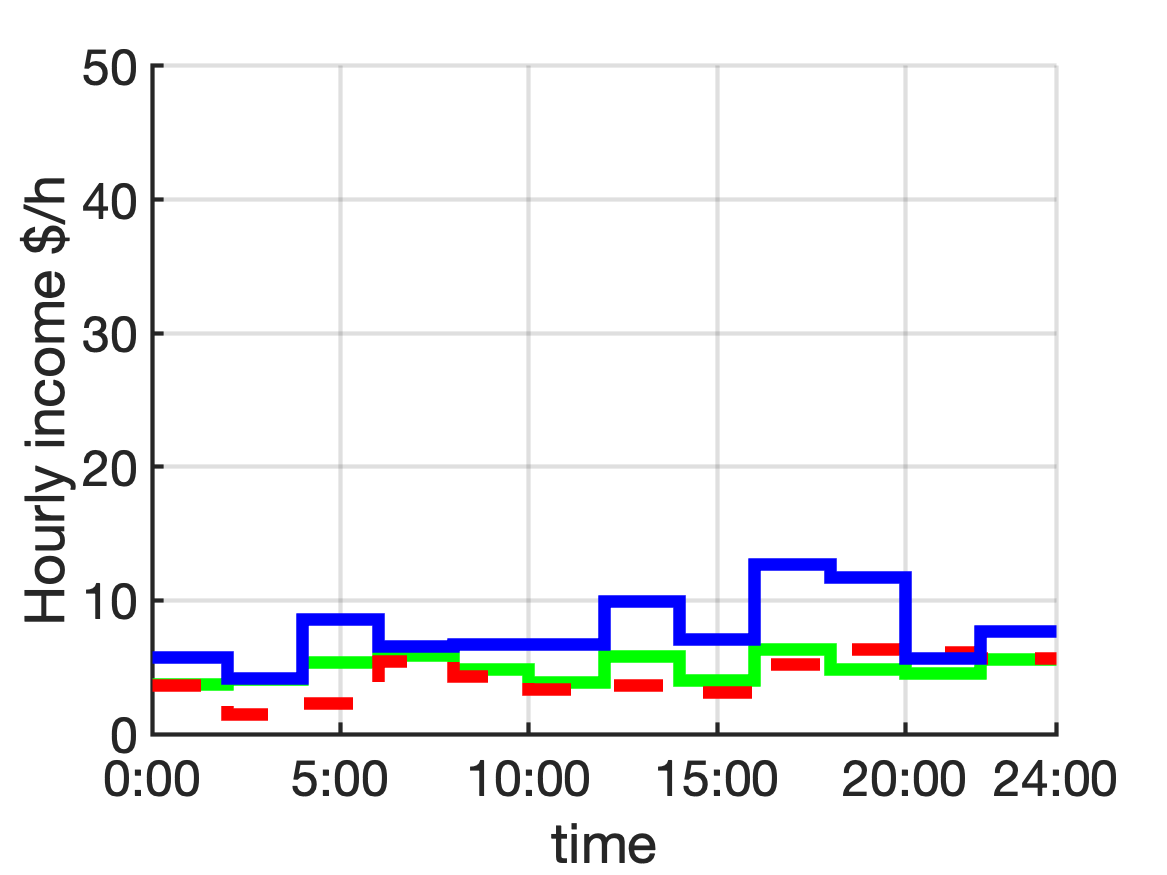}}
   \subfigure[Monday in May 2020]{\includegraphics[width=0.3\linewidth]{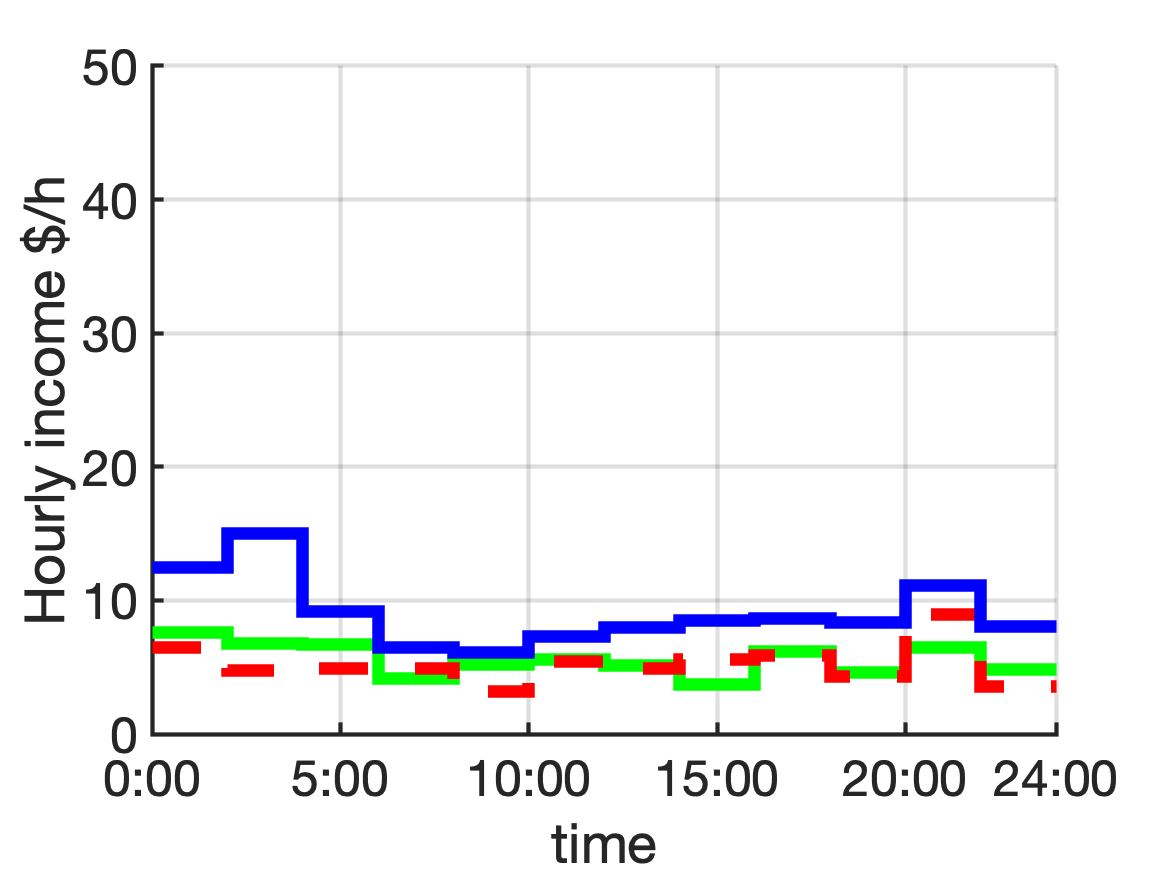}}
   \subfigure[Monday in June 2020]{\includegraphics[width=0.3\linewidth]{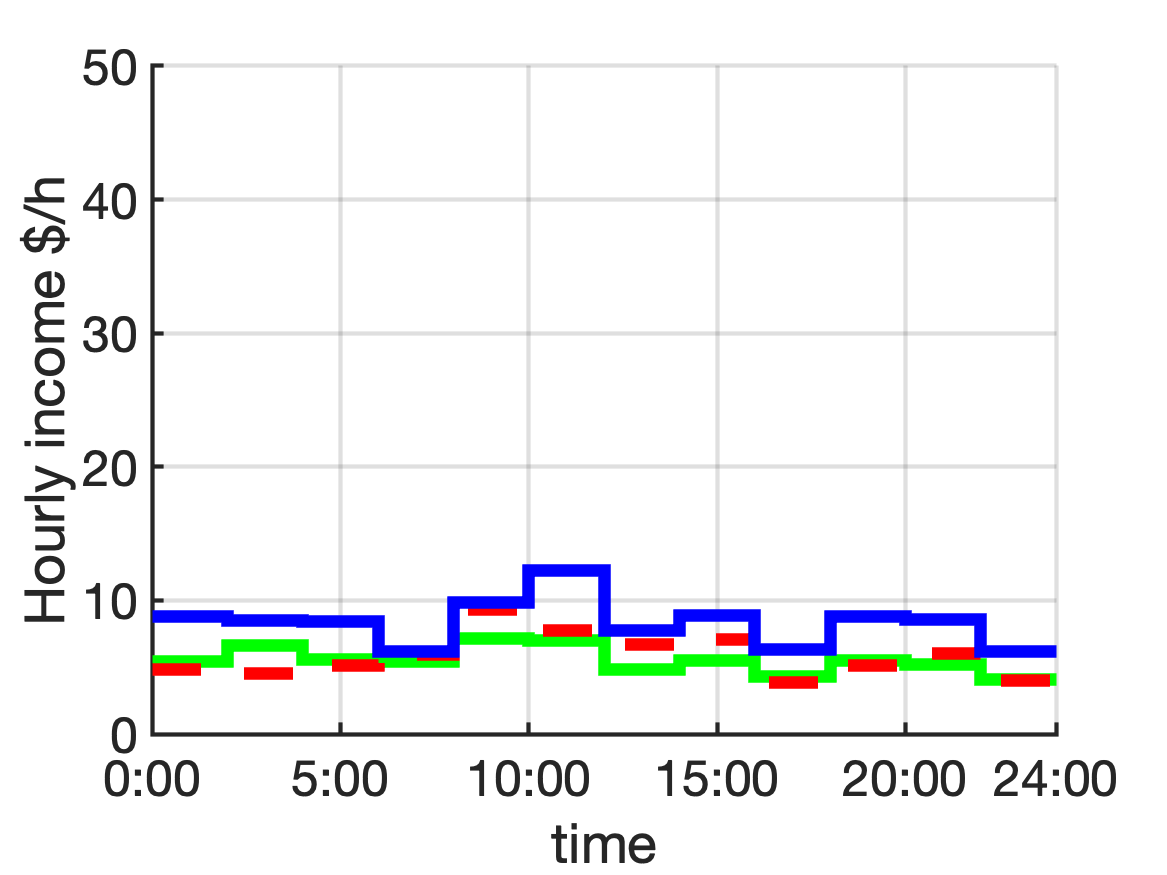}}\\
   \caption{Average Daily Income in Different Months.
   }\label{fig:dailyincome}
\end{figure}

\subsubsection{Hour-by-Hour Income.}
To show more detailed results, we plot the hour by hour income on Monday in \autoref{fig:dailyincome}. 
We plot the results of our method versus the top-two baseline methods based on the recommendation quality: REI and PCD.
For January and March, we find the high-profit hours in the morning from 4 AM to 8 AM; from afternoon to evening, high-profit hours are from 4 PM to 8 PM.
For the time period during COVID-19, the peak hour in the morning disappears, and the peak hours in the evening are also less profitable.
However, our method consistently outperforms the REI and PCD all the time, for virtually all hours and months.

\textbf{Occupancy Rate and Computing Speed}
To evaluate the flexibility of our model, we change the evaluation metric to occupancy rate.
Given the fixed working time for a taxi driver, the smaller the cruising time, the larger the occupancy rate. 
Thus, we define the occupancy rate function as: 

\begin{equation*}
    \frac{\text{delivery~time}}{\text{delivery~time} + \text{cruising~time}}.
\end{equation*}
In this case, the reward in our reinforcement learning model is set to the occupancy rate.
We calculated the average occupancy rate under 30 different routes in June 2020, and the iteration number for the reinforcement learning model is equal to 200,000.
\autoref{fig:case2} (a) presents the average occupancy rate for different baselines in June 2020.
Our method still outperforms all benchmarks, confirming its stability and consistency for different objectives.

We also compare the computing time of our method with the best-performance baseline, REI. 
As shown in \autoref{fig:case2} (b), the computing time of our method (ATDSC) is nearly the same as REI, indicating that the newly proposed self-check mechanism does not introduce a significant computational cost.
Considering the quality improvement (e.g., over 200\% improvement) obtained from our method, the associated new computational cost is small.
\begin{figure}
\centering
   \subfigure[Weekly occupancy rate in June 2020] {\includegraphics[width=0.42\linewidth]{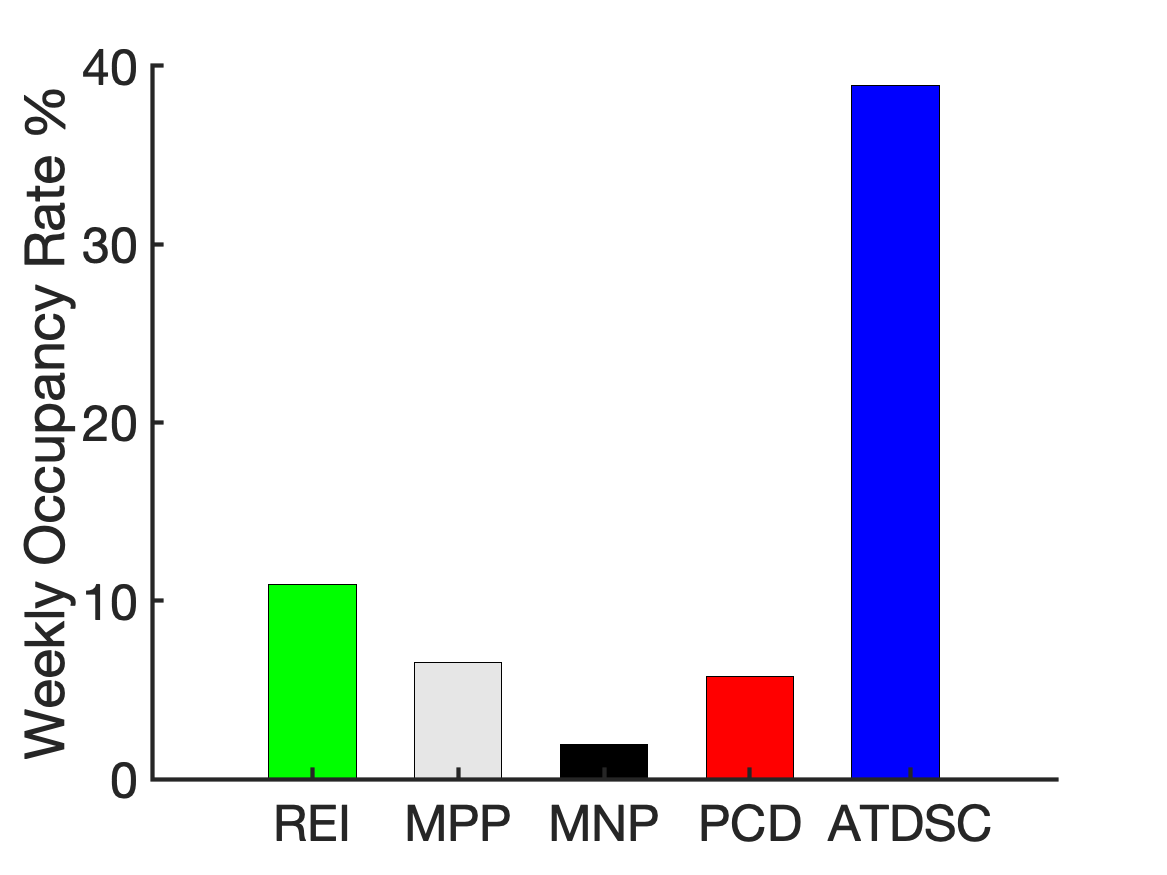}}
   \subfigure[Computing Time]{\includegraphics[width=0.42\linewidth]{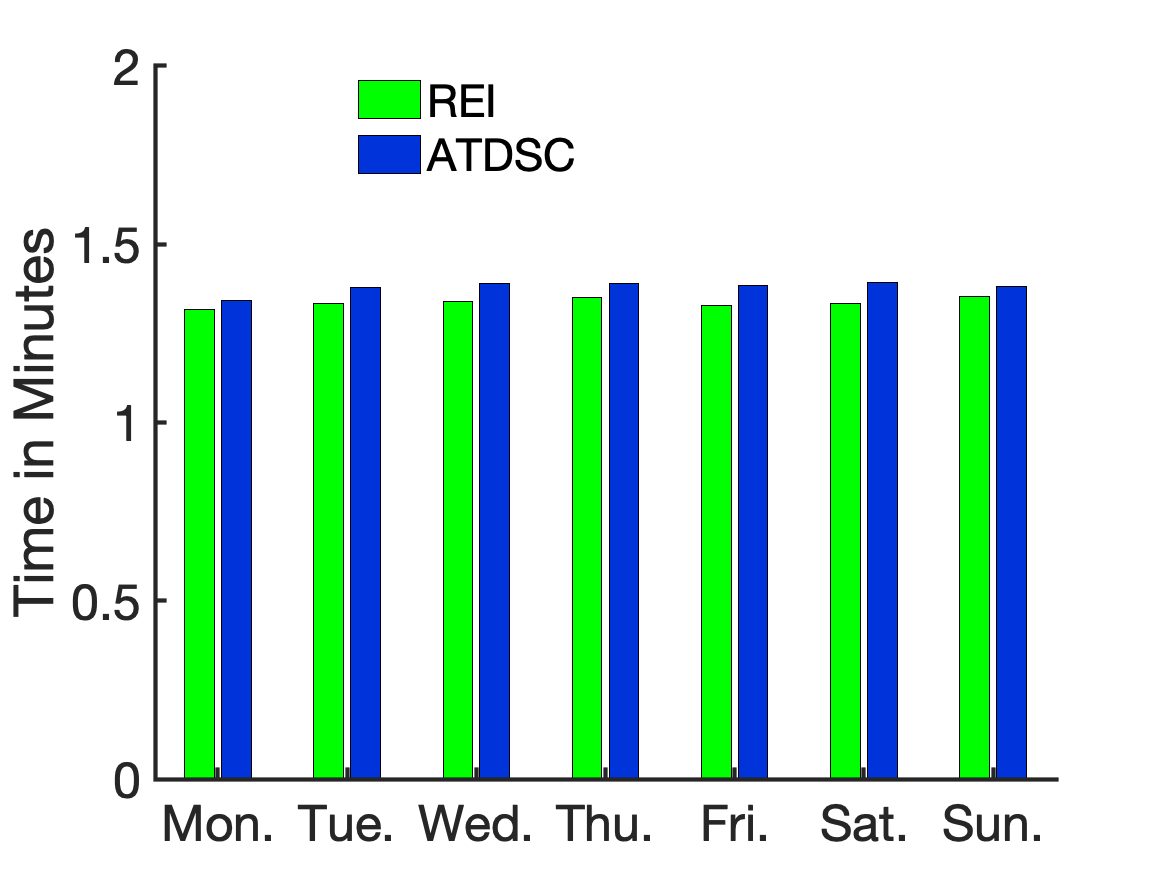}}
   \caption{Occupancy Rate and Computing Time.
   }\label{fig:case2}
\end{figure}

\section{Related work}\label{sec:related work}
Many works have been proposed to recommend a suitable path for taxi drivers based on big data.
The work can be categorized into two groups based on the objective function's flexibility: specific objective functions and flexible objective functions.

\subsection{Specific Objective Functions}
The specific objective function group denotes those papers focusing on optimizing specific goals for the taxi driver. 
To save energy, \citet{ge2010energy} proposes a novel function called the Potential Travel Distance (PTD) function.
Based on the proposed PTD function, they developed two efficient algorithms to find the optimal path.
By reducing the travel distance, they can help drivers save fuel too.
In terms of increasing the number of potential passengers, \citet{yuan2011find} combined the knowledge of passengers’ mobility patterns and taxi drivers’ pick-up behaviors.
Their method not only recommends the point with the highest pick-up probability to drivers, but also guide the passengers to the locations where they can easily find the vacant taxi.
\citet{ye2018multi} introduced a parallel simulated annealing method with domain knowledge to maximize the pick-up probability for several taxi drivers.
By communicating and shuffling the results after fixed steps, they can ensure that the recommended paths are always the global optimal without intersections. Thus their method can achieve a high speedup compared to the sequential method.

\citet{yuan2010t} explored the historical GPS trajectories of a large number of taxis and found the fastest route for the taxi driver to save time for passengers and drivers.
They designed a Variance-Entropy-Based Clustering approach to estimate travel time distribution and constructed the practically fastest route based on the estimation.
To maximize taxi drivers' profit, \citep{qu2014cost} designed a cost-effective recommender system for taxi drivers.
By evaluating the potential profit with a net profit objective function, they can use a greedy method to find the most profitable path for taxi drivers.
\citet{tang2013locating} analyzed large amounts of GPS location data of taxicabs and found a high-level profit-maximizing strategy for taxi drivers.
They treat the problem as a Markov Decision Process (MDP), and the parameters are determined by the historical data.
By applying dynamic programming, they captured meaningful rules on how to find the passenger.
For reducing the total cost in a trip, \citet{zhang2019parallel} developed a simulated annealing-based parallel method.
They spread the best result to each local worker among each communication and ensure that the global optimal can be achieved.
More human trajectory-based objective functions have been formulated in \citep{meng2019hierarchical,liu2016proactive,liu2014proactive, liu2012cocktail}.

\subsection{Flexible Objective Functions}
Compared with a specific objective function case, the flexible objective function focuses more on the model itself.
As long as the objective function is defined as required, the model can update and generate the required result.
\citet{verma2017augmenting} treated the right locations of passengers as a reward to guide the reinforcement learning model.
Based on the basic learning mechanism, they provide a dynamic abstraction mechanism to improve the performance.
The same setting as our method, \citet{gao2018optimize} designed a reinforcement learning-based method to optimize taxi driving strategies by maximizing taxi drivers' profit.
The reward was defined as the effective driving time for a driver within a day,
the results show that the method not only increases the income of a taxi driver, but also helps the passengers find taxi more easily.
\citet{ji2020spatio} designed an effective two-step method with reinforcement learning to deal with the dynamic route recommendation problem.
They define multiple types of rewards: taxi drivers' average earning, taxis' average vacant cruising time, passengers' average waiting time, and passengers being picked up within 30 minutes.


For the above methods, they optimize the objective function based on the definition of reward.
If we want to analyze the effect of other criteria, as long as we change the definition of reward (e.g. from profit to cruising time), we can directly obtain the required result without making any changes to the model.
Thus our proposed method in this work belongs to the flexible objective group.
Those studies with specified objective functions can also be transferred to the flexible objective function by defining the key criteria in their objective function as a variable.
However, the efficiency of the newly defined objective function still needs to be tested.

\section{Conclusion}\label{sec:conclusion and Future WOrk}
In this paper, we proposed an adaptive reinforcement learning method with a self-check mechanism to solve the dynamic route optimization problem. 
Our model can detect irregular events (e.g. public health emergence) and automatically update parameters to adapt to a new traffic environment.  
With a focus on income maximization, the results show that our method can increase at least 98\%
of the average weekly income for taxi drivers under several experimental settings.
We also provide case studies under different evaluation metrics to demonstrate the flexibility and stability of our method.
For future work, our method can be extended to parallel versions for further improvement in computational efficiency. 
It will also be interesting to consider vehicle-sharing scenarios for a more complicated user-based method to satisfy real needs.

\bibliographystyle{unsrtnat}
\bibliography{sample-base}

\appendix

\end{document}